%% file: bare_jrnl_new_sample4.tex
\documentclass[10pt,journal,compsoc]{IEEEtran}
\usepackage{amsmath,amsfonts}
\usepackage{algorithm}
\usepackage{array}
\usepackage[caption=false,font=normalsize,labelfont=sf,textfont=sf]{subfig}
\usepackage{textcomp}
\usepackage{stfloats}
\usepackage{url}
\usepackage{verbatim}
\usepackage{graphicx}
\usepackage{hyperref} 
\usepackage{enumitem}
\usepackage{algpseudocode}
\usepackage{multirow}
\usepackage{relsize}
\usepackage{subcaption}
\usepackage{todonotes}
\usepackage{subfig}
\usepackage{algpseudocode}
\usepackage{booktabs} 
\algnewcommand\algorithmicto{\textbf{to}}
\usepackage{amsthm}
\newtheorem{theorem}{Theorem}

\usepackage[table]{xcolor}
\usepackage{threeparttable}
\usepackage{booktabs}   
\usepackage{balance} 
\usepackage{placeins}
\usepackage{ragged2e}
\ifCLASSOPTIONcompsoc
  \usepackage[nocompress]{cite}
\else
  \usepackage{cite}
\fi
\ifCLASSINFOpdf
\else
\fi

\begin{document}
%
\title{MSDformer: Multi-scale Discrete Transformer For Time Series Generation}
%
%
%
%

\author{Shibo Feng, Zhicheng Chen, Xi Xiao$^\dagger$, Zhong Zhang, \\ Qing Li,~\IEEEmembership{Senior Member, IEEE,}, Xingyu Gao$^\dagger$, Peilin Zhao$^\dagger$~\IEEEmembership{Senior Member, IEEE}
\IEEEcompsocitemizethanks{
\IEEEcompsocthanksitem Shibo Feng is with Nanyang Technological University, Singapore. Email: shibo001@e.ntu.edu.sg.
\IEEEcompsocthanksitem Zhicheng Chen and Xi Xiao are with Shenzhen International Graduate School, Tsinghua University. Email: czc22@mails.tsinghua.edu.cn, xiaox@sz.tsinghua.edu.cn.
\IEEEcompsocthanksitem Zhong Zhang is with Tencent, Shenzhen. Email: todzhang@tencent.com.
\IEEEcompsocthanksitem Qing Li is with Peng Cheng Laboratory, Shenzhen. Email: andyliqing@gmail.com.
\IEEEcompsocthanksitem Xingyu Gao is with the Institute of Microelectronics, Chinese Academy of Sciences, and also with the University of Chinese Academy of Sciences, Beijing, China. Email: gxy9910@gmail.com.
\IEEEcompsocthanksitem Peilin Zhao is with the School of Artificial Intelligence, Shanghai Jiao Tong University, Shanghai, China. Email: peilinzhao@sjtu.edu.cn
\IEEEcompsocthanksitem $\dagger$ Corresponding author.
}
}

%
%

\markboth{Journal of \LaTeX\ Class Files, 2026}%
{Shell \MakeLowercase{\textit{et al.}}: MSDformer: Multi-scale Discrete Transformer For Time Series Generation}
%



\IEEEtitleabstractindextext{%
\begin{abstract}
\justifying
Discrete Token Modeling (DTM), which employs vector quantization techniques, has demonstrated remarkable success in modeling non-natural language modalities, particularly in time series generation. While our prior work SDformer established the first DTM-based framework to achieve state-of-the-art performance in this domain, two critical limitations persist in existing DTM approaches: 1) their inability to capture multi-scale temporal patterns inherent to complex time series data, and 2) the absence of theoretical foundations to guide model optimization. To address these challenges, we proposes a novel multi-scale DTM-based time series generation method, called Multi-Scale Discrete Transformer (MSDformer). MSDformer employs a multi-scale time series tokenizer to learn discrete token representations at multiple scales, which jointly characterize the complex nature of time series data. Subsequently, MSDformer applies a multi-scale autoregressive token modeling technique to capture the multi-scale patterns of time series within the discrete latent space. Theoretically, we validate the effectiveness of the DTM method and the rationality of MSDformer through the rate-distortion theorem. Comprehensive experiments demonstrate that MSDformer significantly outperforms state-of-the-art methods. Both theoretical analysis and experimental results demonstrate that incorporating multi-scale information and modeling multi-scale patterns can substantially enhance the quality of generated time series in DTM-based approaches. Code is available at this repository: \url{https://github.com/kkking-kk/MSDformer}.
\end{abstract}

}

\maketitle

\IEEEdisplaynontitleabstractindextext

%
\IEEEpeerreviewmaketitle

\IEEEraisesectionheading{
\section{Introduction}\label{sec:introduction}}

\IEEEPARstart{T}
{ime} series data is integral to numerous real-world applications, including finance \cite{liu2016online, ding2020hierarchical, yoo2021accurate, xu2021relation,10236709}, healthcare \cite{penfold2013use}, energy \cite{qiu2017short, lim2021time, feng2023multi, feng2024latent}, retail \cite{lim2021temporal, zhou2023one}, and climate science \cite{wu2020adversarial}. Despite its importance, the scarcity of dynamic data often hinders the advancement of machine learning solutions, especially in contexts where data sharing might compromise privacy \cite{alaa2020generative}. To address this challenge, the generation of synthetic yet realistic time series data has gained traction as a viable alternative. This approach has become increasingly popular, driven by recent progress in deep learning methodologies.

The deep learning revolution introduces three principal paradigms for temporal data synthesis:
Generative Adversarial Networks (GANs) \cite{mogren2016c, esteban2017real, yoon2019time, xu2020cot, pei2021towards, jeha2022psa, jeon2022gt}, which pioneer adversarial training frameworks that capture sharp temporal distributions but suffer from non-convex optimization instability and catastrophic mode collapse;
Variational Autoencoders (VAEs) \cite{desai2021timevae, naiman2023generative}, which ensure training stability via evidence lower bound optimization with KL-regularized latent spaces, at the expense of over-smoothed outputs that attenuate high-frequency details;
Denoising Diffusion Probabilistic Models (DDPMs) \cite{kong2020diffwave, coletta2024constrained, yuan2024diffusion}, which achieve state-of-the-art performance through iterative denoising processes but face deployment challenges, including architectural complexity, limited controllability, and computational inefficiency.

In recent years, large transformer-based language models, commonly referred to as LMs or LLMs, have emerged as the predominant solution for natural language generation tasks \cite{achiam2023gpt, anil2023palm}. These models have demonstrated remarkable capabilities in generating high-quality text and have progressively expanded their applications to encompass various modalities beyond natural language. Notably, significant advancements have been made in image generation \cite{ramesh2021zero, yu2022scaling, chang2022maskgit, chang2023muse}, video synthesis \cite{yu2023magvit, villegas2022phenaki,kondratyuk2023videopoet}, and audio generation \cite{zeghidour2021soundstream,copet2023simple,chen2025neural} through the innovative application of Discrete Token Modeling (DTM) technique.
DTM operates by learning discrete representations of diverse data modalities and treating them in a manner analogous to natural language tokens. This technique enables direct adaptation of pre-trained language models for cross-modal generation tasks. The success of these methods has been substantial, yielding impressive results in various multimodal generation tasks.

The remarkable success of the DTM technique in visual and audio domains has motivated its recent extension to time series tasks \cite{zhicheng2024sdformer,feng2025hdt}. Specifically, our prior work SDformer \cite{zhicheng2024sdformer} established the first DTM-based framework that demonstrates statistically significant superiority over GANs, VAEs, and DDPMs in the time series generation task. 
Despite these empirical successes, the theoretical foundations of DTM remain underdeveloped, which limits its methodological advancement and broader application.
Moreover, complex time series often exhibit multi-scale temporal patterns, encompassing both low-frequency trends and short-term fluctuations. Current DTM-based approaches for time series generation, such as SDformer, predominantly utilize single-scale modeling frameworks. This limitation restricts their ability to comprehensively capture the multi-scale temporal patterns inherent in complex time series data.

To address the aforementioned challenges, we propose a novel DTM-based framework for time series generation, termed the Multi-Scale Discrete Transformer (MSDformer). This framework extends SDformer \cite{zhicheng2024sdformer}, advancing from single-scale to multi-scale modeling.
Similar to existing DTM-based approaches, MSDformer operates in a two-stage pipeline. In the first stage, MSDformer constructs a multi-scale time series tokenizer through cascaded residual VQ-VAEs \cite{pang2017cascade,van2017neural}, progressively learning discrete token representations at multiple scales. These multi-scale token sequences collectively capture the comprehensive characteristics of the original time series, enabling a more nuanced representation of both low-frequency and high-frequency features.
In the second stage, MSDformer adopts an autoregressive token modeling technique to learn the distribution of multi-scale token sequences, thereby effectively capturing the multi-scale temporal patterns inherent in time series data. This technique is naturally better suited to sequential data than the masked token modeling technique, as illustrated by SDformer. 
Furthermore, we prove that DTM-based methods offer superior controllability in the rate-distortion trade-off compared to other continuous models, leveraging Shannon’s Rate-Distortion Theorem. Based on this, we further theoretically demonstrate that multi-scale modeling can achieve lower distortion. In summary, our contributions include:
\vspace{0.1in}
\begin{itemize}[leftmargin=*]
    \item \textbf{Multi-scale Generation.} We introduce a novel DTM-based method, MSDformer, which extends SDformer by  an advancement from single-scale to multi-scale modeling. This advancement enables MSDformer to effectively capture the multi-scale temporal patterns inherent in time series data, thereby enhancing generative capabilities.
    \vspace{0.1in}
    \item \textbf{Strong Controllability.} We develop a rate–distortion analysis theoretically demonstrating that discrete tokenization induces an explicit, adjustable rate constraint via finite codebooks, thereby providing stronger controllability than continuous latent-variable models. We further show that, under the same overall rate budget, allocating capacity across scales can reduce distortion more effectively than increasing the codebook size in a single-scale model.
    \vspace{0.05in}
    \item \textbf{Long-term Generation Capability.} Extensive experiments demonstrate that MSDformer consistently outperforms SDformer and other continuous generative models, with particularly clear gains in long-term time-series generation. Compared with SDformer, MSDformer achieves average  improvements of 34.5\% in Discriminative Score and 12.6\% in Context-FID.

\end{itemize}

\section{Related Work}
In this section, we focus on two key directions: 1) advancements in time series generative models, and 2) the applications and developments of the discrete token modeling technique, both of which jointly motivate our framework.
\subsection{Time Series Generation}
\noindent\textbf{Generative Adversarial Networks (GANs)-based Models.} With the continuous advancement of deep generative models, the field of time series generation has achieved remarkable progress.
The initial wave of research predominantly focuses on Generative Adversarial Networks as the primary architecture for time series generation \cite{mogren2016c,esteban2017real,yoon2019time,xu2020cot,pei2021towards,jeha2022psa}. 
For example, TimeGAN \cite{yoon2019time} introduces novel enhancements by incorporating an embedding function and supervised loss to better capture temporal dynamics.
COT-GAN \cite{xu2020cot} innovatively combines the underlying principles of GANs with Causal Optimal Transfer theory, enabling stable generation of both low- and high-dimensional time series data. 

GANs-based models, despite their progress, remain notoriously hard to train due to optimization instability \cite{mescheder2018training, kodali2017convergence, thanhimproving}. They often suffer from non-convergence (oscillatory training), vanishing gradients when the discriminator becomes overly strong, and mode collapse \cite{thanh2020catastrophic}. The latter is especially detrimental for time-series generation \cite{brophy2021generative, wang2024deep}, where the model may degenerate to producing a small set of highly similar sequences (e.g., repetitive periodic patterns), failing to capture the diversity and fine-grained temporal dynamics of real data.


\noindent\textbf{Variational Autoencoders (VAEs)-based Models.} Since VAEs are are known to be more robust to the GANs-based issues above, some attentions are turned to this type of generative paradigm for time series generation. For example, TimeVAE \cite{desai2021timevae}, with its interpretable temporal structure, stands out as a significant advancement, delivering superior performance in time series synthesis. KoVAE \cite{naiman2023generative} integrates Koopman operator theory with VAEs to provide a unified modeling framework for both regular and irregular time series data. 

However, traditional VAEs for time-series generation typically rely on either an unstructured Gaussian prior or structured (e.g., approximately linear) latent dynamics. While these designs can aid inference or dynamical interpretability, they do not necessarily impose explicit constraints aimed at latent controllability, so the learned representations may remain entangled and hard to manipulate. Consequently, controllable generation is limited, and posterior collapse can still occur.

\noindent\textbf{Denoising Diffusion Probabilistic Models (DDPMs).} Denoising Diffusion Probabilistic Models \cite{ho2020denoising} mark a transformative phase in generative modeling, with rapid adaptation to time series generation. Early applications like DiffWave \cite{kong2020diffwave} demonstrate the potential of DDPMs in waveform generation, while more recent innovations such as DiffTime \cite{coletta2024constrained} leverages advanced score-based diffusion models for enhanced temporal data synthesis. The field reaches new heights with Diffusion-TS \cite{yuan2024diffusion}, a diffusion model framework that generates high-quality multivariate time series samples using an encoder-decoder Transformer architecture with disentangled temporal representations. 

While these approaches achieve competitive results, diffusion models typically map non-i.i.d. time-series distributions to an isotropic Gaussian prior via iterative noise
injection, producing long and complex trajectories. This may hinder sampling efficiency and limit performance,  especially in long-term generation setting, which are shown in our following experimental section.


\subsection{Discrete Token Modeling}
In contrast to continuous generative models above, which operate in a non-explicit constrained latent space or require costly iterative sampling, discrete token modeling strategy provides a new generation paradigm. Discrete Token Modeling (DTM) is a technique for modeling discrete token sequences using sequence models, originally widely applied in natural language modalities \cite{devlin2019bert}. In comparison, GAN- and VAE-based generators often struggle to provide reliable and fine-grained control: GANs are prone to unstable training and mode collapse, while VAEs may yield entangled latents and suffer from posterior collapse, limiting controllable synthesis. Diffusion models improve sample quality but typically require dozens to hundreds of denoising steps, which makes long-horizon generation computationally expensive. By discretizing time series into an explicit token space, DTM offers a more controllable and efficient alternative, where  generation is performed in an explicit, finite code space that is easier to constrain, manipulate, and sample efficiently. In recent years, the rise of large-scale language models has driven a new era of development in artificial intelligence \cite{achiam2023gpt, anil2023palm}, highlighting the efficiency and effectiveness of the DTM technique in sequence data. Moreover, DTM has been extended to non-natural language modalities—such as iamge \cite{ramesh2021zero}, video \cite{yan2021videogpt}, and audio \cite{du2022vqtts} through vector quantization models like VQ-VAE \cite{van2017neural} and VQ-GAN \cite{esser2021taming}.



Early discrete token modeling (DTM) in non-language modalities mainly follows two paradigms: Autoregressive Token Modeling (ARTM) and Masked Token Modeling (MTM). ARTM predicts the next token from previous tokens using a categorical distribution, and has been widely used in text-to-image generation (e.g., DALL-E \cite{ramesh2021zero}, Parti \cite{yu2022scaling}) as well as text-to-motion synthesis (e.g., T2M-GPT \cite{zhang2023t2m}, MotionGPT \cite{jiang2024motiongpt}). In contrast, MTM adopts a masked prediction objective \cite{devlin2018bert}, reconstructing randomly masked tokens from the observed context, which underlies models such as MaskGIT \cite{chang2022maskgit} and MUSE \cite{chang2023muse} for efficient and effective image generation. More recently, Visual AutoRegressive Modeling (VAR) \cite{tian2024visual} introduces a new DTM paradigm that reframes autoregression as coarse-to-fine ``next-scale '' prediction: images are tokenized into multi-scale token maps via hierarchical quantization, and generation proceeds autoregressively from low to high resolutions by predicting the next higher-resolution token map, improving global coherence while preserving fine local details. Also, this type of multi-scale modeling perspective is also naturally compatible with time-series data \cite{chen2023multi, chen2024pathformer, feng2023multi}, where temporal dynamics always exhibit hierarchical structure.

Inspired by the broad applicability of the DTM technique across modalities, researchers have begun exploring its potential for time series tasks \cite{zhicheng2024sdformer,feng2025hdt}. 
HDT \cite{feng2025hdt} proposes a hierarchical DTM-based framework that achieves high accuracy in high-dimensional, long-term time series forecasting tasks.
Our prior work, SDformer \cite{zhicheng2024sdformer}, introduces the DTM technique to time series generation and proposes a similarity-driven vector quantization method, achieving a significant breakthrough in generation quality.

However, our previous SDformer adopts a single-scale, similarity-based discrete modeling strategy in the latent space. While effective, it overlooks the inherent multi-scale nature of time-series data and lacks a systematic theoretical and empirical investigation of why discrete modeling is beneficial. In this work, we demonstrate that MSDformer consistently improves short- and long-term generation, outperforming SDformer as well as representative continuous generative baselines. Moreover, we find that multi-scale token modeling yields larger gains in generation quality than simply increasing the codebook size at a single scale, highlighting the advantage of explicitly modeling temporal hierarchies.

\section{Problem Statement}
Consider a multivariate time series as a sequence of observations denoted by $x_{1:\tau}=(x_1,\cdots,x_\tau)\in \mathbb{R}^{\tau\times d}$, where $\tau$ represents the number of time steps and $d$ denotes the dimensionality of variables at each time step. In practical applications, we typically work with a collection of such time series. Formally, a dataset containing $n$ time series can be represented as $D=\{x_{1:\tau}^i\}_{i=1}^n$, where each $x_{1:\tau}^i$ corresponds to an independent realization of the underlying temporal process.

The task of time series generation aims to learn the intrinsic distribution of the observed dataset $D$ and subsequently generate new, realistic time series samples that preserve the statistical properties of the original dataset. Mathematically, this can be formulated as learning a generative model $f_\theta$ that maps samples from a known prior distribution to the time series domain:
\begin{equation}
    \hat{x}_{1:\tau}^i = f_\theta(Z),
\end{equation}
where $Z$ is typically sampled from a simple distribution (e.g., Gaussian distribution) and serves as the source of randomness for generation.

\section{Methods}
\input{figure/model}
\input{figure/transformer}
In this section, we introduce MSDformer, a novel multi-scale DTM-based model tailored for time series generation. MSDformer is a two-stage framework designed to learn and model multi-scale information and patterns in time series data.
In the first stage, MSDformer employs a multi-scale time series tokenizer that leverages similarity-driven vector quantization to generate high-quality discrete token representations across multiple scales, as depicted in Fig. \ref{fig:model}. Building on these multi-scale discrete token representations, the second stage utilizes a multi-scale autoregressive token modeling technique to capture the multi-scale patterns of time series in the discrete latent space, as illustrated in Fig. \ref{fig:transformer}.

\subsection{Multi-scale Time Series Tokenizer}

The multi-scale tokenizer implements cascaded residual learning \cite{pang2017cascade} through stacked VQ-VAE models \cite{van2017neural}, where each subsequent model hierarchically encodes residual information from the preceding scale. This coarse-to-fine decomposition progressively refines temporal features, yielding more precise and detailed representations of the underlying time series data. Algorithm \ref{alg:vqvae} outlines the training procedure for the multi-scale time series tokenizer in this stage, in which our objective is to train a multi-scale time series tokenizer that can produce high-quality discrete representations of time series data while ensuring excellent reconstruction performance. 

Let there be $ K $ scales, each with an encoder $\mathcal{E}^{(k)}$, a decoder $\mathcal{D}^{(k)}$, a codebook $ C^{(k)} $, and a downsampling factor $ r^{(k)} $, where $ k = 1, \ldots, K $. The learnable codebook $C^{(k)} = \{c_v\}_{v=0}^{V^{(k)}-1} \subset \mathbb{R}^{d_c}$ at each scale contains $ V^{(k)} $ latent embedding vectors, each of dimension $ d_c $. The downsampling factor $ r^{(k)}\geq 1 $ decreases as $ k $ increases, meaning that smaller $ k $ processes larger scale information.

The input to the model at each scale is the residual $ f_{1:\tau}^{(k)} \in \mathbb{R}^{\tau \times d} $. When $ k = 1 $, $ f_{1:\tau}^{(k)} $ is initialized as $ x_{1:\tau} $. For $ k > 1 $, $ f_{1:\tau}^{(k)} $ is the residual between the input of the previous scale and its reconstruction output.
Specifically, starting from $ k = 1 $, the following operations are performed:

\textbf{(1) Encoding Process.} First, the encoder $\mathcal{E}^{(k)}$ initially applies 1D convolutions to the input $f_{1:\tau}^{(k)}$ along the temporal dimension, producing encoded vector:
\begin{equation}
    h_{1:L^{(k)}}^{(k)}=\mathcal{E}^{(k)}(f_{1:\tau}^{(k)}),
    \label{eq:encoder}
\end{equation}
where $h_{1:L^{(k)}}^{(k)} \in \mathbb{R}^{L^{(k)} \times d_c}$ is the encoded feature, and $L^{(k)} = \tau / r^{(k)}$ is the length in the discrete latent space at the $k$-th scale.

\textbf{(2) Vector Quantization.} Then, the similarity-driven vector quantization $Q(\ \cdot \ ;C^{(k)})$ identifies the index of the vector in the codebook $C^{(k)}$ that exhibits the highest similarity to $h_i^{(k)}$, which can be expressed as:
\begin{equation}
    y_{i}^{(k)}=Q(h_i^{(k)};C^{(k)}):=\arg\max_{v}h_i^{(k)}\cdot c_v^{(k)},
\label{Eq:quantization}
\end{equation}
where $\cdot$ denotes the inner product, and the discrete token sequence of the $k$-th scale can be represented as $y_{1:L^{(k)}}^{(k)} \in \prod_{\ell=1}^{L^{(k)}} \{0, 1, \ldots, V^{(k)}-1\}$, with $\prod$ denoting the Cartesian product.

\textbf{(3) Dequantization Process.} Following quantization, the dequantization process $Q^{-1}(\ \cdot \ ;C^{(k)})$ reverts $y_i^{(k)}$ back to the latent embedding vector, denoted as:
\begin{equation}
    \tilde{h}_{i}^{(k)}=Q^{-1}(y_{i}^{(k)};C^{(k)}):=c_{y_{i}^{(k)}}^{(k)},
\label{Eq:dequantization}
\end{equation}

\textbf{(4) Decoding Process.} Subsequently, the decoder $\mathcal{D}^{(k)}$ restores it to the raw time series space:
\begin{equation}
    \tilde{f}_{1:\tau}^{(k)}=\mathcal{D}^{(k)}(\tilde{h}_{1:L^{(k)}}^{(k)}).
    \label{eq:decoder}
\end{equation}

\textbf{(5) Residual Calculation.} Lastly, we obtain the residual input for the next scale, denoted as:
\begin{equation}
    f_{1:\tau}^{(k+1)} = f_{1:\tau}^{(k)} - \tilde{f}_{1:\tau}^{(k)}.
    \label{eq:residual}
\end{equation}
After completing the above steps, continue to repeat the above operations of modeling the residual $f_{1:\tau}^{(k+1)}$ at the $(k+1)$-th scale until the $K$-th scale, following steps (1)-(5).

After processing all scales as described above, the final reconstruction of the time series is achieved by aggregating the reconstruction results from all scales:
\begin{equation}
    \tilde{x}_{1:\tau} = \sum_{k=1}^{K} \tilde{f}_{1:\tau}^{(k)}.
    \label{eq:reconstruction}
\end{equation}

To train this multi-scale time series tokenizer, we utilize two distinct loss functions for training and optimizing the parameters of $\mathcal{E}^{(k)}$ and $\mathcal{D}^{(k)}$:

\begin{equation}
    \mathcal{L} = \left\| x_{1:\tau} - \tilde{x}_{1:\tau} \right\|_2^2 
    + \lambda \sum_{k=1}^K \frac{1}{L^{(k)}} \sum_{i=1}^{L^{(k)}} \left( 1 - h_{i}^{(k)} \cdot \text{sg}(\tilde{h}_{i}^{(k)}) \right),
\label{Eq:vqloss}
\end{equation}
where the first loss is the reconstruction loss, the second loss is the embedding loss, sg$(\cdot)$ represents the stop gradient, and $\lambda$ is hyperparameter used to adjust the weight of different parts. 

For the codebook vectors, an Exponential Moving Average technique \cite{williams2020hierarchical} is used for updating. Specifically, the $y_{i}^{(k)}$-th vector $ c_{y_{i}^{(k)}}^{(k)} $ in the codebook $C^{(k)}$, which is used in Equation \eqref{Eq:quantization} through vector quantization, is updated as follows:

\begin{equation}
    \tilde{c}_{y_{i}^{(k)}}^{(k)}= (1-\beta) c_{y_{i}^{(k)}}^{(k)} + \beta h_i^{(k)},
    \label{Eq:ema}
\end{equation}
where $\tilde{c}_{y_{i}^{(k)}}^{(k)}$ represents the updated result, and $\beta > 0$ is a hyperparameter that adjusts the weight, typically chosen to be a small value to ensure the stability of the update.

The codebook update mechanism in Equation \eqref{Eq:ema} introduces potential optimization stagnation: vectors initialized far from encoded feature clusters may never receive gradient updates. To mitigate this, we implement a Codebook Reset strategy \cite{williams2020hierarchical} that dynamically reinitializes underutilized vectors. 
The average usage count $u_v^{(k)} $ for each codebook vector $c_v^{(k)} $ is maintained through exponential smoothing:
\begin{equation}
    \tilde{u}_v^{(k)}  = (1-\beta)u_v^{(k)}  + \beta n_v^{(k)} ,
    \label{eq:code_usage_rate}
\end{equation}
where $ n_v^{(k)}  $ counts the current batch usage of $ c_v^{(k)}  $, and $ \tilde{u}_v^{(k)}  $ is the updated average usage count. In the Codebook Reset technique, whether $ c_v^{(k)}  $ is reset can be expressed as:
\begin{equation}
    \tilde{c}_v^{(k)}  = \begin{cases}
        e_j \sim \mathcal{U}(H^{(k)}), & \text{if } \tilde{u}_v^{(k)}  < u_{th} \\
        c_v^{(k)} , & \text{otherwise}
    \end{cases},
    \label{Eq:reset}
\end{equation}
where $ H^{(k)} $ contains encoded features from the current batch via encoder $ \mathcal{E}^{(k)} $, and $ \mathcal{U}(H^{(k)}) $ denotes uniform sampling from $ H^{(k)} $. 
The reset threshold $u_{th}=1$ ensures replacement of vectors averaging less than one usage per batch.

When the multi-scale time series tokenizer training is completed, all codebooks and parameters will be frozen. By employing this multi-scale time series tokenizer, a multivariate time series $x_{1:\tau} \in \mathbb{R}^{\tau \times d}$ can be mapped to $K$ time series tokens sequences $\left\{ y^{(1)}_{1:L^{(1)}}, \ldots, y^{(k)}_{1:L^{(k)}}, \ldots, y^{(K)}_{1:L^{(K)}} \right\}$.

\begin{algorithm}[h!]
    \caption{Training of multi-scale time series tokenizer.}
    \label{alg:vqvae}
     \textbf{Input:} Time series dataset $D=\{x_{1:\tau}^i\}_{i=1}^n$\\ 
    \textbf{Output:} Multi-scale Encoder $\{\mathcal{E}^{(k)}\}_{k=1}^K$, multi-scale Decoder $\{\mathcal{D}^{(k)}\}_{k=1}^K$ and multi-scale codebooks $\{C^{(k)}\}_{k=1}^K$.\\
    \vspace{-0.1in}
    \begin{algorithmic}[1]
    \For{$i\leftarrow 1$ to $N_{iter}$}
        \State Get the $X_{1:\tau}$ $\sim$ $D$; 
        \State Initialize residual input $f_{1:\tau}^{(1)}\leftarrow X_{1:\tau}$;
        \For{$k\leftarrow 1$ to $K$}
            \State Get the $h_{1:L^{(k)}}^{(k)}$ by Equation \eqref{eq:encoder};
            \State Get the $y_{1:L^{(k)}}^{(k)}$ by Equation \eqref{Eq:quantization};
            \State Get the $\tilde{h}_{1:L^{(k)}}^{(k)}$ by Equation \eqref{Eq:dequantization};
            \State Get the $\tilde{f}_{1:\tau}^{(k)}$ by Equation \eqref{eq:decoder};
            \State Get the $f_{1:\tau}^{(k+1)}$ by Equation \eqref{eq:residual};
        \EndFor
        \State Get the $\tilde{x}_{1:\tau}$ by Equation \eqref{eq:reconstruction};
        \State Compute the training loss $\mathcal{L}$ by Equation \eqref{Eq:vqloss};
        \State Complete backpropagation process based on $\mathcal{L}$ and update the parameters of $\{\mathcal{E}^{(k)}\}_{k=1}^K$ and $\{\mathcal{D}^{(k)}\}_{k=1}^K$;
        \State Update the codebook by Equations \eqref{Eq:ema} and \eqref{Eq:reset};
    \EndFor
    \State Return trained $\{\mathcal{E}^{(k)}\}_{k=1}^K$, $\{\mathcal{D}^{(k)}\}_{k=1}^K$, and $\{C^{(k)}\}_{k=1}^K$.
    \end{algorithmic}
\end{algorithm}

\subsection{Multi-scale Autoregressive Token Modeling}
In this subsection, we focus on utilizing an autoregressive token modeling technique to learn the data distribution of the time series in the multi-scale discrete token space. In this stage, we aim to train a multi-scale autoregressive discrete Transformer that learns the distribution of multi-scale discrete token sequences, and the training process is shown in Algorithm~\ref{alg:train_ms_ar}. Besides, for the training algorithms for our SDformer, simply set the number of scales $K$ in MSDformer to 1.

\textbf{Token Sequence Construction.} First, the range of discrete tokens at the $ k $-th scale is $\{0, 1, \ldots, V^{(k)}-1\}$. Since token indices at different scales start from 0, the same token index may have different meanings across scales. To distinguish these tokens, the range of discrete tokens at each scale is adjusted to ensure uniqueness across scales. Specifically, let $ S^{(k)} := \sum_{m=1}^k V^{(m)} $ denotes the cumulative vocabulary size up to the $ k $-th scale, where $ S^{(0)} := 0 $. For each scale $ k \in \{1, 2, \ldots, K\} $, the original token $ y^{(k)}_i \in \{0, 1, \ldots, V^{(k)}-1\} $ is shifted to ensure its uniqueness:
\begin{equation}
    \bar{y}^{(k)}_i = y^{(k)}_i + S^{(k-1)} \in \{S^{(k-1)}, S^{(k-1)}+1, \ldots, S^{(k)}\}.
    \label{eq:bias}
\end{equation}
This approach ensures that the token range at each scale is unique, avoiding ambiguity in token meanings across scales.

After processing the token ranges for all scales, the token sequences from all scales are concatenated in order to form the final token sequence:
\begin{equation}
    y_{1:L} = \bar{y}^{(1)}_{1:L^{(1)}} \oplus \bar{y}^{(2)}_{1:L^{(2)}} \oplus \ldots \oplus \bar{y}^{(K)}_{1:L^{(K)}},
    \label{eq:label}
\end{equation}
where $ L = \sum_{k=1}^K L^{(k)} $, and $ \oplus $ denotes concatenation along the time dimension.

\textbf{Input Sequence Construction.} By shifting $ y_{1:L} $ and concatenating the [BOS] token, the input sequence is constructed as:
\begin{equation}
    y^{\text{in}}_{1:L} = \text{[BOS]} \oplus y_{1:L-1},
    \label{eq:input}
\end{equation}
where the index of the [BOS] token is set to $ V = \sum_{k=1}^{K} V^{(k)} $ to ensure its uniqueness in the multi-scale token sequence.

Additionally, for data augmentation, each element $ \bar{y}^{(k)}_i $ in the input sequence is randomly replaced with probability $ \epsilon $:
\begin{equation}
    \tilde{y}^{(k)}_i = 
    \begin{cases}
        \text{Randint}(S^{(k-1)}, S^{(k)}), & \text{if } \xi \leq \epsilon \\
        \bar{y}^{(k)}_i, & \text{otherwise}
    \end{cases},
    \label{Eq:random_replacement}
\end{equation}
where $\text{Randint}(S^{(k-1)}, S^{(k)})$ denotes a random integer uniformly sampled from the range $ S^{(k-1)} $ to $ S^{(k)} - 1 $, $ \mathcal{U}(0,1) $ represents a uniform distribution over the interval $[0, 1]$, and $\xi\sim\mathcal{U}(0,1)$ is a random number sampled from this uniform distribution. Finally, the model takes the augmented sequence $ \tilde{y}^{\text{in}}_{1:L} $ as the final input sequence.

\begin{algorithm}[h!]
    \caption{Training of  multi-scale autoregressive discrete Transformer.}\label{alg:train_ms_ar}
    \textbf{Input:} Time series dataset $D=\{x_{1:\tau}^i\}_{i=1}^n$, optimized multi-scale time series tokenizer.  \\ 
    \textbf{Output:}Multi-scale autoregressive Discrete Transformer $\mathcal{G}_{ms}$. \\
    \vspace{-0.1in}
    \begin{algorithmic}[1]
    \For{$i\leftarrow 1$ to $N_{iter}$}
        \State Feed the $x_{1:\tau}$ to multi-scale time series tokenizer and get $\left\{ y^{(1)}_{1:L^{(1)}}, \ldots, y^{(k)}_{1:L^{(k)}}, \ldots, y^{(K)}_{1:L^{(K)}} \right\}$.
        \State Get the label $y_{1:L}$ by Equations \eqref{eq:bias} and \eqref{eq:label}; 
        \State Get the input $\tilde{y}_{1:L}^{in}$ by Equations \eqref{eq:input} and \eqref{Eq:random_replacement};
        \State Feed the $\tilde{y}_{1:L}^{in}$ to $\mathcal{G}_{ms}$;
        \State Compute the training loss $\mathcal{L}_{ms}$ by Equation \eqref{eq:multi-scale_ar_loss};
        \State Complete backpropagation process based on $\mathcal{L}_{ms}$ and update the parameters of $\mathcal{G}_{ms}$;
    \EndFor
    \State Return trained $\mathcal{G}_{ms}$.
    \end{algorithmic}
\end{algorithm}

\textbf{Multi-scale Autoregressive Discrete Transformer.} 
In this part, we construct a multi-scale autoregressive discrete Transformer model, denoted as $\mathcal{G}_{ms}$. The input sequence to $\mathcal{G}_{ms}$ is $\tilde{y}^{\text{in}}_{1:L}$, and the prediction target is $y_{1:L}$. Since $y_{1:L}$ is arranged in order from coarse to fine scales, $\mathcal{G}_{ms}$ first generates data reflecting trends and overall structures in the time series, followed by progressively refining the details. The architecture utilizes a standard decoder-only Transformer framework \cite{vaswani2017attention}, with its core difference residing in the design of the input embedding representation. Specifically, the embedding representation of a standard autoregressive discrete Transformer is composed of the sum of Token Embedding and Positional Encoding:
\begin{equation}
    h_i^{\text{base}} = \mathbf{E}_t[\tilde{y}_i^{\text{in}}] + \mathbf{E}_p[i],
    \label{eq:base_embedding}
\end{equation}
where $\mathbf{E}_t \in \mathbb{R}^{(V+1) \times d_m}$ is the token embedding matrix, $d_m$ is the hidden dimension of the Transformer, and $\mathbf{E}_p \in \mathbb{R}^{L \times d_m}$ is the positional encoding matrix. To model multi-scale tokens, $\mathcal{G}_{ms}$ additionally introduces Token Type Encoding:
\begin{equation}
    h_i^{\text{ms}} = \underbrace{\mathbf{E}_t[\tilde{y}_i^{\text{in}}]}_{\text{Token Embedding}} + \underbrace{\mathbf{E}_p[i]}_{\text{Positional Encoding}} + \underbrace{\mathbf{E}_s[k_i]}_{\text{Token Type Encoding}},
    \label{eq:ms_embedding}
\end{equation}
where $\mathbf{E}_s \in \mathbb{R}^{(K+1) \times d_m}$ is the token type encoding matrix, which uses to explicitly indicate which scale each token comes from. Since tokens from different scales are concatenated into a single input sequence, this embedding prevents the Transformer from confusing tokens of different temporal resolutions and enables scale-aware attention and prediction, $K$ is the number of scales, and the additional dimension is reserved for the [BOS] token. Here, $k_i \in \{0, 1, \ldots, K\}$ is the scale identifier for the token, defined as follows.

Given $L^{(k)}$ as the number of tokens at the $k$-th scale, the cumulative token count $T^{(k)}$ is defined as:
\begin{equation}
    T^{(k)} = \sum_{m=1}^k L^{(m)}, \quad T^{(0)} = 0.
    \label{eq:sum_code}
\end{equation}
The assignment rule for the token type index $k_i$ is given by:
\begin{equation}
    k_i = 
    \begin{cases}
        0, & \text{if } i = 0 \\
        k \text{ s.t. } T^{(k-1)} < i \leq T^{(k)}, & \text{otherwise}
    \end{cases}.
    \label{eq:scale_assignment_final}
\end{equation}
Specifically, the [BOS] token corresponds to $k_i = 0$, while the $k_i$ for other tokens is uniquely determined by the scale interval $[T^{(k-1)}+1, T^{(k)}]$ in which their index $i$ falls. The corresponding token type encoding sequence structure can be represented as:
\begin{equation}
    \underbrace{0,}_{i=0}\ 
    \underbrace{1, \dots, 1,}_{i=1 \text{ to } T^{(1)}}\ 
    \underbrace{2, \dots, 2,}_{i=T^{(1)}+1 \text{ to } T^{(2)}}\ 
    \dots \ 
    \underbrace{K, \dots, K.}_{i=T^{(K-1)}+1 \text{ to } T^{(K)}}
    \label{eq:type-emb-seq}
\end{equation}

In summary, the multi-scale autoregressive discrete Transformer model $\mathcal{G}_{ms}$ consists of the embedding representation layer described above, along with standard Transformer components such as multi-head self-attention mechanisms and feed-forward neural network modules. The output at each position is a probability distribution, with the output probability distribution at the $i$-th position defined as:
\begin{equation}
    P_{\mathcal{G}_{ms}}(y_i | \tilde{y}^{\text{in}}_{1:i}) = \text{Softmax}\left( \mathcal{G}_{ms}(\tilde{y}^{\text{in}}_{1:i}) \right).
    \label{eq:output2}
\end{equation}

\noindent\textbf{Training Objective.}
The training objective is to minimize the negative log-likelihood of all tokens:
\begin{equation}
    \mathcal{L}_{ms} = -\mathbb{E} \left[ \sum_{i=1}^L \log P_{\mathcal{G}_{ms}}(y_i | \tilde{y}^{\text{in}}_{1:i}) \right].
    \label{eq:multi-scale_ar_loss}
\end{equation}
In the specific implementation, the loss function is computed using cross-entropy.
\subsection{Inference Stage of MSDformer} 
\vspace{-0.5mm}
Once the multi-scale autoregressive discrete Transformer is trained, it can be used to generate time series data. This process consists of three steps: autoregressive token generation, multi-scale token assignment, and multi-scale decoding, as outlined in Algorithm~\ref{alg:msdformer}. Specifically, the process begins with the [BOS] token as the initial input, followed by the autoregressive generation of multi-scale discrete token sequences using the multi-scale autoregressive discrete Transformer. For the generated multi-scale discrete token sequences, the tokens are assigned back to their corresponding scales and restored to their original token value ranges. Finally, the original time series data is reconstructed through the dequantization process and the decoding process of the multi-scale time series tokenizer, thereby achieving the generation functionality.
To ensure the diversity of the generated data, each generated token is sampled from the predicted distribution provided by the Transformer.

\begin{algorithm}[h!]
\caption{The Inference Process of MSDformer.}
\label{alg:msdformer}
\begin{algorithmic}[1]
\Require
\State Target sequence length $\tau$, 
\State Multi-scale downsampling factors $\{r^{(k)}\}_{k=1}^K$, 
\State Multi-scale codebooks $\{C^{(k)}\}_{k=1}^K$, 
\State Multi-scale codebook sizes $\{V^{(k)}\}_{k=1}^K$, 
\State Trained multi-scale decoders $\{\mathcal{D}^{(k)}\}_{k=1}^K$, 
\State Multi-scale autoregressive discrete Transformer $\mathcal{G}_{ms}$

\State \textbf{Step 1: Autoregressive Token Generation}
\State $V \gets \sum_{k=1}^K V^{(k)}$ \Comment{Total vocabulary size}
\State $L \gets \sum_{k=1}^K \tau/r^{(k)}$ \Comment{Total sequence length}
\State $\hat{y}_{0} \gets V$ \Comment{Initialize with BOS token}
\For{$t = 1$ \algorithmicto \hspace{1pt} $L$}
    \State $\mathcal{Z} \gets \mathcal{G}_{ms}(\hat{y}_{0:t-1})$ \Comment{Generate token logits}
    \State $p \gets \text{Softmax}(\mathcal{Z}[-1])$ \Comment{Compute probabilities}
    \State $\mathcal{I} \gets \text{Sample}(p)$ \Comment{Sample next token}
    \State $\hat{y}_{0:t} \gets \hat{y}_{0:t-1} \oplus \mathcal{I}$ \Comment{Update sequence}
\EndFor

\State \textbf{Step 2: Multi-scale Token Allocation}
\State $i \gets 1$ \Comment{Initialize pointe}
\State $S^{(0)} \gets 0$ \Comment{cumulative vocabulary size}
\For{$k = 1$ \algorithmicto \hspace{1pt} $K$}
    \State $S^{(k)} \gets S^{(k-1)} + V^{(k)}$ \Comment{Update vocabulary offset}
    \State $L^{(k)} \gets \tau/r^{(k)}$ \Comment{Calculate scale-specific length}
    \State $\hat{y}^{(k)}_{1:L^{(k)}} \gets \hat{y}_{i:i+L^{(k)}} - S^{(k-1)}$ \Comment{Allocate tokens}
    \State $i \gets i + L^{(k)}$ \Comment{Move pointer}
\EndFor

\State \textbf{Step 3: Multi-scale Decoding}
\For{$k = 1$ \algorithmicto \hspace{1pt} $K$}
    \State $\hat{h}_{1:L^{(k)}}^{(k)} \gets Q^{-1}(\hat{y}^{(k)}_{1:L^{(k)}};C^{(k)})$\Comment{Dequantize tokens}
    \State $\hat{f}_{1:\tau}^{(k)} \gets \mathcal{D}^{(k)}(\hat{h}_{1:L^{(k)}}^{(k)})$ \Comment{Decoding}
\EndFor
\State $\hat{x}_{1:\tau} \gets \sum_{k=1}^K \hat{f}_{1:\tau}^{(k)}$ \Comment{Aggregate multi-scale outputs}
\Ensure Generated time series $\hat{x}_{1:\tau} \in \mathbb{R}^{\tau \times d}$

\end{algorithmic}
\end{algorithm}

\vspace{-0.1in}
\section{Theoretical Analysis}
\subsection{Theoretical Analysis of Discrete Token Modeling}
\label{sec:theorem_DTM}
In this section, we present a theoretical analysis of Discrete Token Modeling for time series generation, establishing its fundamental advantages over continuous latent variable approaches in achieving explicit rate-distortion control through finite codebook constraints.

\begin{theorem}[Shannon's Rate-Distortion Theorem \cite{shannon1959coding}]\label{thm:rate_distortion}
    \label{theorem:shannon}
    As an information-theoretic lens, consider a fixed-length time-series window (block) random vector
    \(X_{1:\tau} \sim P_{data}\) and a bounded distortion measure
    \(d:\mathcal{X}^{\tau}\times\tilde{\mathcal{X}}^{\tau}\to\mathbb{R}^{+}\).
    The minimum achievable rate at distortion level \(D\) is characterized by the
    (block) rate--distortion function: 
    \begin{equation}
        R{(D)}=\frac{1}{\tau} \inf_{\substack{p(\tilde{X}_{1:\tau}\,|\,X_{1:\tau}):\\ \mathbb{E}[d(X_{1:\tau},\tilde{X}_{1:\tau})]\le D}}
        I\!\left(X_{1:\tau};\tilde{X}_{1:\tau}\right),
    \end{equation}
    where \(I(X_{1:\tau};\tilde{X}_{1:\tau})\) denotes the mutual information between the window \(X_{1:\tau}\)
    and its reconstruction \(\tilde{X}_{1:\tau}\). \(\mathbb{E}[\cdot]\) denotes expectation.
\end{theorem}

For the time series generation task, let $X_{1:\tau}=(X_1,\ldots,X_\tau)\in\mathbb{R}^{\tau\times d}$
denote a random window sampled from the data distribution $P_{data}$ and 
$\tilde{X}_{1:\tau}$
for its reconstruction through $P_{\mathcal{E},\mathcal{D}}(X_{1:\tau}| \theta_\mathcal{E}, \theta_{\mathcal{D}})$ of stage 1. The distortion is defined as the window MSE:
\begin{equation}
d(X_{1:\tau},\tilde{X}_{1:\tau})
=\frac{1}{\tau}\sum_{t=1}^{\tau}\|X_t-\tilde{X}_t\|^2.
\end{equation}
In the vector quantization method, the time series data \(X_{1:\tau}\) is compressed into a discrete token sequence $y_{1:L} \in \{0, 1, \ldots, V-1\}^L$ via an encoder $\mathcal{E}$ of stage 1. The rate-distortion control for this method is described in Theorem \ref{thm:vq_control}.

\begin{theorem}[Vector Quantization Rate-Distortion Control]
\label{thm:vq_control}
For any distortion level \(D > 0\), there exists \(V^*\) such that whenever the codebook size \(V \geq V^*\), there exists a coding scheme ensuring:
\[
\mathbb{E}\!\left[d(X_{1:\tau},\tilde{X}_{1:\tau})\right]\le D.
\]
\end{theorem}
\begin{proof}
By Theorem \ref{theorem:shannon}, for any \( D > 0 \), 
there exists a rate \( R(D) \) such that if the coding rate \( R \geq R(D) \), there exists a coding scheme achieving \( \mathbb{E}\left[\,d(X_{1:\tau}, \tilde{X}_{1:\tau})\,\right] \leq D \). 
In vector quantization, the codebook size \( V \) and discrete sequence length \( L \) determine the rate:
\begin{equation}
    R = \log_2 V^L = L \log_2 V.
    \label{eq:rate}
\end{equation}
Define \( V^* = \lceil2^{R(D)/L}\rceil \). When \( V \geq V^* \), we have:
\begin{equation}
    \begin{aligned}
        V \geq 2^{R(D)/L} &\implies L \log_2 V \geq R(D)\\
        &\implies R \geq R(D)
    \end{aligned}.
\end{equation}
Thus, there exists a coding scheme such that:
\begin{equation}
\mathbb{E}_{X_{1:\tau}\sim P_{data}, \tilde{X}_{1:\tau} \sim P_{\mathcal{E},\mathcal{D}}(X_{1:\tau}| \theta_\mathcal{E}, \theta_{\mathcal{D}})}[d(X_{1:\tau}, \tilde{X}_{1:\tau})] \le D.
\end{equation}
\end{proof}

According to Theorem \ref{thm:vq_control}, in the vector quantization method, the codebook size \( V \) directly controls the rate \( R = L \log_2 V \). Based on rate-distortion theory, there is a trade-off between the rate \( R \) and the distortion \( D \):
\begin{itemize}
    \item Increasing \( V \) directly raises the rate \( R \), which in turn allows for lower distortion.
    \item When \( R \geq R(D) \), the theory ensures the existence of a coding scheme that meets the target distortion \( D \).
\end{itemize}

In contrast, the rate-distortion behavior of continuous latent variable models, such as Variational Autoencoders (VAEs), is inherently uncontrollable. Specifically, VAEs are trained by maximizing the Evidence Lower Bound (ELBO):
\begin{equation}
    \mathcal{L}_{ELBO} = \mathbb{E}_{q_\phi(z|x)}[\log p_\theta(x|z)] - \beta D_{\mathrm{KL}}(q_\phi(z|x)\|p(z)),
\end{equation}
where the rate \( R \) is determined by the mutual information of the latent variable \( z \):
\begin{equation}
    I(X;Z) \approx \mathbb{E}_{p(x)}\left[D_{\mathrm{KL}}(q_\phi(z|x)\|p(z))\right].
\end{equation}
\begin{equation}
R_{vae} = \frac{I(X;Z)}{\ln 2}.
\end{equation}
Since \( q_\phi(z|x) \) is a continuous distribution, the mutual information \( I(X;Z) \) cannot be explicitly constrained and may be unstable (e.g., posterior collapse), thus failing to ensure \( R \geq R(D) \) because \( I(X;Z) \) may not reach the threshold. 

\subsection{Theoretical Analysis of Multi-scale Modeling}
\label{sec:theorem_MSM}
In Section \ref{sec:theorem_DTM}, we conclude that increasing the codebook size directly raises the rate, which in turn allows for lower distortion. In MSDformer, the multi-scale modeling, which involves learning multi-scale discrete token representations of time series, indirectly increases the codebook size. 
To compare the merits of indirectly increasing the codebook size by adding scales versus directly increasing the codebook size on a single scale, this part will conduct a comparative analysis from a theoretical perspective.

We assume the original discrete sequence length is \( L \) and the codebook size is \( V \). For a fair comparison, we assume a fixed increase in codebook size \( V' \).
In the setting of directly increasing the codebook size on a single scale, the rate \( R \) increases from Equation \eqref{eq:rate} to:
\begin{equation}
\begin{aligned}
    R_s &= \log_2 (V + V')^L\\
    &=L\log_2V\left(1+\frac{V'}{V}\right)\\
    &=L\log_2V+L\log_2\left(1+\frac{V'}{V}\right)
\end{aligned}.
\end{equation}
In the setting of increasing scale, we use \( V' \) as the codebook size for the new scale. Additionally, the length of the discrete token sequence increases by \( L' \) in the new scale. In this case, the rate \( R \) increases from Equation \eqref{eq:rate} to:
\begin{equation}
    \begin{aligned}
        R_m &= \log_2 \left( V^L{V'}^{L'} \right)\\
        &=L\log_2V+L'\log_2V'
    \end{aligned}.
\end{equation}
In our experimental setup, the ratios \( \frac{V'}{V} \) and \( \frac{L'}{L} \) are constrained to the interval \(\left[\frac{1}{4}, 4\right]\), and they share the same directional relationship (both are either $\geq1$ or $\leq 1$). Furthermore,  \( V \) and \( V' \) are set within the range \([2^7, 2^{10}]\). 
Under these conditions, it is observed that \( L \log_2 \left(1 + \frac{V'}{V}\right) < L' \log_2 V' \), which leads to the conclusion that \( R_m > R_s \) always holds true.
Therefore, in general, employing multi-scale modeling is more effective in enhancing the rate compared to directly increasing the codebook size on a single scale, thereby allowing for lower distortion. This theoretically validates the rationality of the MSDformer design.

\section{Experiments}
In this section, we begin by evaluating our proposed method through a comparative analysis with several state-of-the-art baseline methods on short-term and long-term time series generation tasks. Following this, we validate the efficiency advantages of DTM-based methods through efficiency analysis, and the superiority of key components and designs of MSDformer through ablation studies. Finally, we elucidate the operational mechanism of MSDformer in learning multi-scale temporal features through visualization analysis.

\subsection{Datasets and Metrics}

\noindent\textbf{Datasets.} We evaluate the proposed method on six benchmark datasets, including four real-world datasets (Stocks, ETT, Energy, and fMRI) and two simulated datasets (Sines and MuJoCo), which are popular introduced in the previous works~\cite{yoon2019time, zhicheng2024sdformer ,yuan2024diffusion}, and the details of datasets are shown in Appendix A. 

\noindent\textbf{Metrics.} To quantitatively evaluate the synthetic data, we employ three metrics: 1) The \textbf{Discriminative Score} assesses distributional similarity by measuring how well a classifier can distinguish between real and synthetic samples; 2) The \textbf{Predictive Score} evaluates the practical utility of synthetic data by evaluating the performance of models trained on synthetic samples when tested against real data; 3) The \textbf{Context-FID Score} quantifies sample fidelity through representation learning, specifically comparing the distribution of local temporal features between real and synthetic time series. The details of calculations are shown in Appendix A and all experiments report mean and standard deviation over 5 random seeds.

We employ two visualization techniques to assess the fidelity of synthetic data: 1) \textbf{t-SNE} \cite{van2008visualizing} for nonlinear dimensionality reduction, which projects both original and synthetic samples into a shared 2D embedding space to reveal structural similarities; and 2) \textbf{Kernel Density Estimation} (KDE) to generate comparative probability density plots, enabling a direct visual assessment of distributional alignment between real and synthetic data across feature dimensions.

\noindent\textbf{Implementation Details.} For Stage\textbf{1}: we set the codebook sizes and dimensions in $\{128, 256, 512\}$ and $K$ with $\{2, 4, 8\}$ across all datasets. For Stage\textbf{2}: we use the decoder-only Transformer with 2 layers, and set $\epsilon$ within $\{0.1, 0.3\}$. All experiments are conducted on a single Nvidia V-100 GPU with the AdamW\cite{loshchilov2017decoupled} optimizer. More details of our model and configurations are put in Appendix B, duo to page limits.

\subsection{Short-term Time Series Generation}
\label{sec:short}
\input{table/experiment}
\input{table/long}

Table \ref{tab:experiments} summarizes the performance of each method across all datasets. 
The results in bold indicate the best performance, while the results with an underline indicate the second-best performance. These results allow us to draw several key observations.

The results in Table \ref{tab:experiments} demonstrate that both MSDformer and SDformer, consistently outperform other methods based on GAN, VAE, and DDPMs. Specifically, they achieve an average improvement of over 60\% in Discriminative Score and over 80\% in Context-FID Score, highlighting the effectiveness of DTM for time series generation tasks. 
Furthermore, MSDformer demonstrates additional improvements over SDformer, largely due to its ability to learn multi-scale patterns in time series data, which highlights the benefits of multi-scale modeling over single-scale modeling.

Fig.~\ref{fig:vis_tsne} and Fig.~\ref{fig:vis_kernel567} of the Appendix show the comparative visualization results of MSDformer, SDformer, and Diffusion-TS across all datasets. These visualizations depict the distribution of synthesized and real time series data  in a two-dimensional space and their probability density distributions. The results indicate that the data generated by MSDformer and SDformer closely resemble real data, highlighting the superior performance of DTM-based methods in time series generation tasks.
Furthermore, through a detailed comparison, it can still be observed that the performance of MSDformer is better than that of SDformer. For example, in the Stock dataset, when the x-axis value is around 0.7, the kernel density estimation of the time series data generated by SDformer has a slight deviation from the kernel density estimation of the real data, while MSDformer is significantly closer to the real data at this position. This further illustrates the advantages of multi-scale modeling.

\subsection{Long-term Time Series Generation}
\label{sec:long}
To thoroughly evaluate our method's effectiveness in generating longer time series, we conduct a comprehensive comparison of various models on long-term time series generation tasks, as detailed in Table \ref{tab:long}. Results show that many methods face significant distortions in longer sequences, with Discriminative Scores nearing 0.5, especially in the Energy dataset. This highlights the increased difficulty of long sequence generation compared to short sequences.

Despite this, SDformer and MSDformer maintain high performance, thanks to DTM's strong ability to learn high-quality discrete representations and its robust sequence modeling capabilities. This validates the significant advantages of DTM-based methods in handling longer sequences. 
Furthermore, MSDformer outperforms SDformer, with an average improvement of 34.5\% in Discriminative Score and 12.6\% in Context-FID Score. This is because long sequences often contain more pronounced multi-scale patterns, and MSDformer has a stronger ability to model these patterns.



\subsection{Efficiency Analysis}
\input{table/time}
To validate the practical efficiency of different methods, we compare their inference speeds, as shown in Table \ref{tab:time}. As can be seen from the table, DTM-based methods, namely SDformer and MSDformer, significantly outperform DDPMs-based methods such as Diffusion-TS in terms of inference speed, with a difference of more than 10 times. This is because DDPMs-based methods require multiple denoising steps, leading to longer inference times. Additionally, MSDformer requires more inference time than SDformer. This is because multi-scale modeling, while learning higher-quality discrete representations of time series, also increases the length of the discrete token sequence, thereby increasing the time required for autoregressive generation.
Despite the moderate speed cost, MSDformer's significant quality gains in both short-term and long-term generation (as shown in Section \ref{sec:short} and Section \ref{sec:long}) justify this trade-off for most real-world applications. The choice between SDformer and MSDformer can be adapted to specific latency-vs-quality requirements. 
\subsection{Ablation Studies}
To understand the contribution of each design or component to proposed method, we conduct ablation experiments for six aspects:
1) the effect of increasing codebook size vs. number of scales;
2) the effect of different scale configuration; 
3) the effect of similarity-driven vector quantization;
4) the effect of exponential moving average for codebook update;
5) the effect of codebook reset;
6) the effect of token type encoding.

\textbf{Effect of Increasing Codebook Size vs. Number of Scales.}
Section \ref{sec:theorem_MSM} demonstrates that increasing the number of scales allows for lower distortion compared to increasing the codebook size. Furthermore, we provide an experimental performance comparison of these approaches under different length settings on the fMRI dataset, as shown in Table \ref{tab:codebook_expansion}. 
From the comparison, it is evident that using SDformer with a codebook size of \( V=512 \) as the baseline, increasing the codebook size to \( V=640 \) results in performance improvements. However, the performance enhancement is more pronounced and stable when increasing the number of scales (e.g., \( V^{(1:K)}=[128,512] \), i.e., MSDformer). These findings align with the theoretical analysis in Section \ref{sec:theorem_MSM}, collectively validating the advantages of multi-scale modeling.

\input{table/codebook_expansion}

\textbf{Effect of Different Scale Configuration.}
Our method's key innovation over SDformer is the integration of multi-scale modeling. This part analyzes the impact of multi-scale modeling and different scale sizes. Table \ref{tab:scale} summarizes MSDformer's performance on the fMRI dataset under various scale configurations, highlighting several key points:

Under single-scale settings ($K=1$), optimal temporal downsampling factors vary across  datasets. For instance, the fMRI dataset performs best with a downsampling factor of $r=2$, while Stocks is optimal at $r=4$. This reflects the differing importance of information at various scales across datasets. Excessively large downsampling factors, like $r=8$, consistently degrade performance, as seen in the fMRI and Stocks datasets.
These factors result in significant performance degradation in both Discriminative Score and Context-FID Score compared to the optimal single-scale configuration.

Multi-scale modeling significantly boosts performance when properly configured. For fMRI dataser, all multi-scale setups ($K\geq2$) outperform single-scale baselines, particularly in Discriminative Score and Context-FID Score. The configuration $r^{(1:K)}=[2,8]$ excels among single and dual-scale cases, balancing local and global features. The triple-scale setup $r^{(1:K)}=[2,4,8]$ offers stable, robust performance across metrics, confirming the utility of each scale.

For Stocks dataset, only the Discriminative Score consistently improves over single-scale configurations in dual-scale setups, while Context-FID Score does not. The best performance is achieved with $r^{(1:K)}=[2,4]$, as it integrates useful information from both scales. Including $r=8$ does not enhance performance, likely due to limited large-scale information in the dataset.

Thus, increasing the number of scales does not inherently improve performance; it depends on the dataset and data characteristics. Positive gains are only achieved by adding scales that contain useful information.

\input{table/scale}

\textbf{Effect of Similarity-Driven Vector Quantization.}
Vector quantization critically impacts discrete token learning in time series modeling. 
Following the finding that similarity-driven strategy outperforms distance-based strategy in time series generation tasks  \cite{zhicheng2024sdformer}, we investigate whether MSDformer retains this advantage by replacing its similarity-driven vector quantization with distance-based variant (``w/o similarit'' in Table \ref{tab:ablation}). Results show that MSDformer with similarity-driven vector quantization consistently surpasses distance-based counterparts across all metrics, aligning with SDformer's finding. this superiority stems from the enhanced capacity of similarity metrics to capture directional features (e.g., trends, periodicity), yielding higher-quality discrete representations that better preserve temporal patterns during modeling.

\textbf{Effect of Exponential Moving Average for Codebook Update.}
MSDformer employs exponential moving average (EMA) instead of standard gradient-based codebook updates in VQ-VAE \cite{van2017neural}. To validate this design, we compare EMA with conventional methods (``w/o EM'' in Table \ref{tab:ablation} denotes gradient-updated variant). Results reveal consistent performance degradation across metrics when using gradient updates, confirming EMA's critical role. The EMA strategy ensures stable codebook vector updates while preserving long-term temporal patterns through momentum-based optimization, thereby enhancing model stability and representation fidelity.

\textbf{Effect of Codebook Reset.}
MSDformer's codebook reset mechanism reactivates underused vectors to prevent codebook collapse. Ablation experiments  show disabling this technique (``w/o Reset'' in Table \ref{tab:ablation}) causes severe performance degradation: Discriminative Score rises from 0.005 to 0.423 and Context-FID increases from 0.009 to 14.664. Codebook utilization plummets to 1.1\% without reset, indicating most vectors remain inactive during training. This collapse directly limits representational capacity, confirming the reset's critical role in maintaining codebook diversity through dynamic vector reactivation.

\textbf{Effect of Token Type Encoding.}
MSDformer addresses multi-scale feature confusion through Token Type Encoding (TTE). Ablation experiments reveal disabling TTE (``w/o TTE'' in Table~\ref{tab:ablation}) causes 40\% Discriminative Score degradation and 22.2\% Context-FID increase, confirming feature entanglement occurs without explicit scale differentiation. This demonstrates TTE's critical role in maintaining scale-specific feature spaces during multi-scale temporal modeling. By explicitly indicating the scale ID of Eqn~\ref{eq:ms_embedding}, TTE provides a lightweight yet essential inductive bias for multi-scale discrete time-series modeling, which enables our MSDformer to achieve better generation performance.
\input{table/ablation}

\subsection{Visual Analysis}
\input{figure/vis_ms}
The experimental results demonstrate that incorporating multi-scale information and establishing multi-scale patterns modeling in the DTM-based time series generation framework lead to significant performance improvements. To elucidate the operational mechanism of MSDformer in learning multi-scale temporal features, we conduct a dedicated analysis of the relationship between the original time series and the multi-scale token sequences generated by the multi-scale time series tokenizer.

Specifically, we conduct a controlled experiment using the Energy dataset with $K=2$, corresponding to two distinct scales. After the model converges, component-wise decoding is performed by feeding each scale-specific token sequence into its corresponding decoder module. The reconstructed results for each scale, along with the complete reconstruction and the original time series, are visualized and compared, as shown in Fig. \ref{fig:vis_ms}.

As illustrated in Fig. \ref{fig:vis_ms}, the following observations can be made:
Firstly, the coarse-scale decoder effectively reconstructs the global trend trajectories and periodic structures, capturing the overall shape and long-term dynamics of the time series.
Secondly, the fine-scale decoder focuses on capturing high-frequency signal components and local pattern details, thereby adding finer textures and short-term fluctuations to the reconstruction.
Finally, the integrated reconstruction combines both resolution levels, merging the global trends and local details to achieve optimal fidelity. This combined approach ensures that the reconstructed time series closely matches the original data in both structure and detail.

\section{Conclusion}
This paper introduces a novel DTM-based method, MSDformer, which extends SDformer by incorporating multi-scale modeling, thereby significantly enhancing its ability to handle multi-scale patterns. Additionally, the superiority of the DTM technique in time series generation tasks is demonstrated through the lens of rate-distortion theory. We also provide a theoretical analysis of the advantages of multi-scale modeling.
Furthermore, experimental results reveal that DTM-based methods, specifically MSDformer and SDformer, consistently outperform continuous models such as GANs, VAEs, and DDPMs in time series generation (TSG) task. The experiments also highlight the advantages of multi-scale modeling, with MSDformer showing notable improvements over SDformer. Notably, MSDformer exhibits particularly clear gains on long-term generation task, where modeling long-range dependencies and cross-scale structures becomes more challenging. These results indicate that the proposed multi-scale discrete token hierarchy better preserves global dynamics while refining local variations over extended generation length.

Looking ahead, future work could focus on enhancing scalability by exploring adaptive multi-scale structures that dynamically adjust the number of scales and downsampling rates across temporal resolutions. We also plan to explore a unified discrete tokenization strategy for cross-domain generation, enabling a single model to learn transferable token representations across diverse time-series domains and to adapt efficiently to unseen domains. In addition, the framework could be extended to more complex scenarios, such as spatiotemporal data generation and multivariate conditional generation, to further investigate its cross-modal potential. This study not only advances time-series generation but also provides insights for other temporal analysis tasks, such as forecasting and anomaly detection.

\bibliographystyle{IEEEtran}
\bibliography{references.bib} 

%

\newpage

\appendices
\input{supp}

\end{document}

%% file: figure/model.tex
\begin{figure*}[t!]
\begin{center}
\includegraphics[width=1\textwidth]{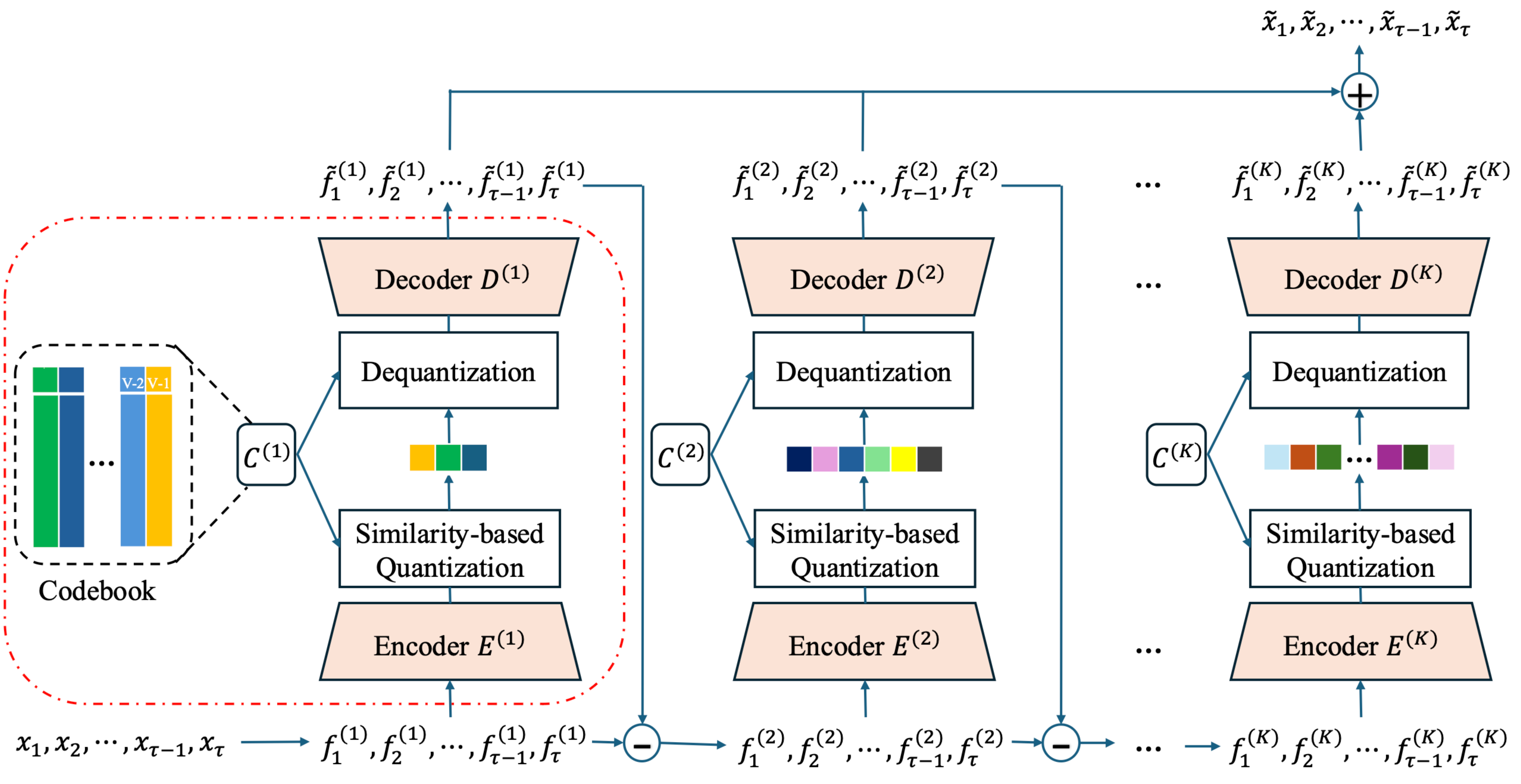}
\end{center}
\caption{Workflow of the multi-scale time-series tokenizer. The tokenizer consists of (K) modules, each implemented as a VQ-VAE with similarity-driven vector quantization to learn discrete tokens at a specific scale. The first module captures the coarse, large-scale structure of the series, while each subsequent module models the residual (i.e., the difference between the input and the reconstruction from the previous scale). Token sequences from all scales are then aggregated to represent the full time series. When (K=1), the tokenizer reduces to the SDformer \cite{zhicheng2024sdformer}, indicated by the red dashed line.}
\label{fig:model}
\end{figure*}


%% file: figure/transformer.tex

\begin{figure}[t!]
\begin{center}
\includegraphics[width=0.5\textwidth]{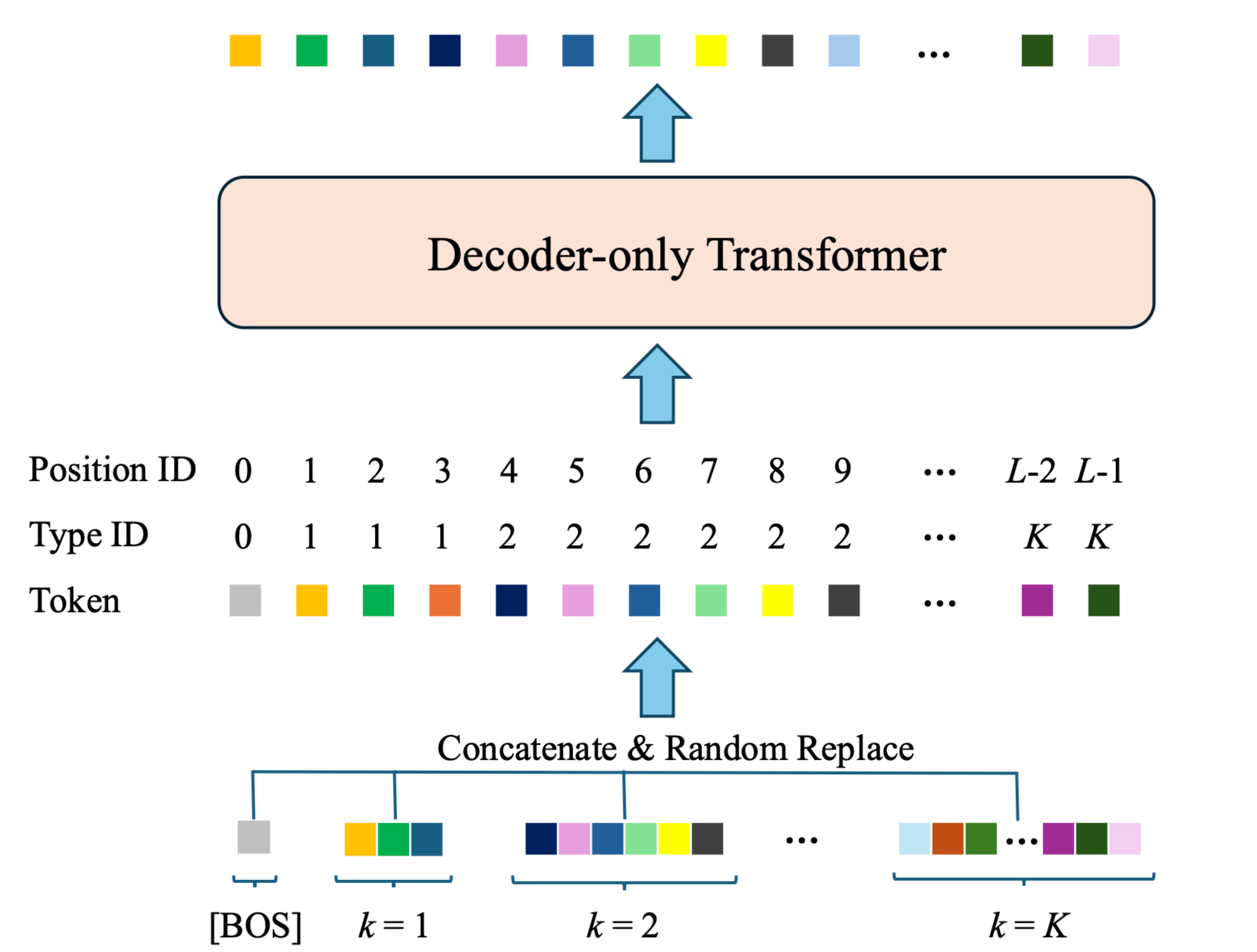}
\end{center}
\caption{Stage-2 workflow of MSDformer (multi-scale autoregressive token modeling). A [BOS] token is prepended to the concatenation result of each scale’s discrete token sequence, and the resulting tokens with type and position IDs are fed into a decoder-only Transformer to autoregressively predict the next token. When (K=1), it reduces to ourSDformer’s single-scale token modeling \cite{zhicheng2024sdformer}.
}
\label{fig:transformer}
\end{figure}

%% file: table/experiment.tex
\begin{table*}[t]
\footnotesize
\relsize{-0.5}
\renewcommand{\arraystretch}{1.3}
\centering
\caption{Results for all methods on all datasets are averaged over five runs (mean ±.standard deviation).}
\resizebox{1.0\linewidth}{!}{
    \begin{tabular}{c|c|c|c|c|c|c|c}
    \hline
    Metrics &Methods &Sines &Stocks &ETTh &MuJoCo &Energy &fMRI\\ 
    \hline
    \multirow{7}{*}{\parbox{1.5cm}{\centering Discriminative\\Score$\downarrow$}}
    &MSDformer &\textbf{0.004±.002} &\textbf{0.005±.004} &\textbf{0.003±.002} &\textbf{0.004±.002} &\textbf{0.005±.003} &\textbf{0.005±.004}\\
    &SDformer &\underline{0.006±.004} &\underline{0.010±.006} &\textbf{0.003±.001} &\underline{0.008±.005} &\underline{0.006±.004} &\underline{0.017±.007} \\ 
    
    &Diffusion-TS &\underline{0.006±.007} &0.067±.015 &\underline{0.061±.009} &\underline{0.008±.002} &0.122±.003 &0.167±.023\\
    &TimeGAN &0.011±.008 &0.102±.021 &0.114±.055 &0.238±.068 &0.236±.012 &0.484±.042 \\
    &TimeVAE &0.041±.044 &0.145±.120 &0.209±.058 &0.230±.102 &0.499±.000 &0.476±.044\\
    &Diffwave &0.017±.008 &0.232±.061 &0.190±.008 &0.203±.096 &0.493±.004 &0.402±.029\\
    &DiffTime &0.013±.006 &0.097±.016 &0.100±.007 &0.154±.045 &0.445±.004 &0.245±.051\\
    &Cot-GAN &0.254±.137 &0.230±.016 &0.325±.099 &0.426±.022 &0.498±.002 &0.492±.018\\
    \hline
    \multirow{8}{*}{\parbox{1.5cm}{\centering Predictive\\Score$\downarrow$}}
    &MSDformer &\textbf{0.093±.000} &\underline{0.037±.000} &\textbf{0.118±.002} &\textbf{0.007±.001} &\textbf{0.249±.000} &\textbf{0.089±.003}\\
    &SDformer &\textbf{0.093±.000} &\underline{0.037±.000} &\textbf{0.118±.002}  &\textbf{0.007±.001} &\textbf{0.249±.000} &\underline{0.091±.002}\\
    &Diffusion-TS &\textbf{0.093±.000} &\textbf{0.036±.000} &\underline{0.119±.002} &\textbf{0.007±.000} &\underline{0.250±.000} &{0.099±.000}\\
    &TimeGAN &\textbf{0.093±.019} &0.038±.001 &0.124±.001 &0.025±.003 &0.273±.004 &0.126±.002\\
    &TimeVAE &\textbf{0.093±.000} &0.039±.000 &0.126±.004 &0.012±.002 &0.292±.000 &0.113±.003\\
    &Diffwave &\textbf{0.093±.000} &0.047±.000 &0.130±.001 &0.013±.000 &0.251±.000 &0.101±.000\\
    &DiffTime &\textbf{0.093±.000} &0.038±.001 &0.121±.004 &\underline{0.010±.001} &0.252±.000 &0.100±.000\\
    &Cot-GAN &\underline{0.100±.000} &0.047±.001 &0.129±.000 &0.068±.009 &0.259±.000 &0.185±.003\\
    \cline{2-8}
    &Original &0.094±.001 &0.036±.001 &0.121±.005 &0.007±.001 &0.250±.003 &0.090±.001\\
    \hline
    \multirow{7}{*}{\parbox{1.5cm}{\centering Context-FID\\Score$\downarrow$}} 
    &MSDformer &\textbf{0.001±.000} &\textbf{0.002±.000} &\textbf{0.002±.000} &\underline{0.006±.001} &\textbf{0.003±.000} &\textbf{0.009±.001}\\
    &SDformer &\textbf{0.001±.000} &\textbf{0.002±.000} &\underline{0.008±.001} &\textbf{0.005±.001} &\textbf{0.003±.000} &\underline{0.015±.001}\\
    &Diffusion-TS &\underline{0.006±.000} &0.147±.025 &0.116±.010 &{0.013±.001} &\underline{0.089±.024} &0.105±.006\\
    &TimeGAN &0.101±.014 &\underline{0.103±.013} &0.300±.013 &0.563±.052 &0.767±.103 &1.292±.218 \\
    &TimeVAE &0.307±.060 &0.215±.035 &0.805±.186 &0.251±.015 &1.631±.142 &14.449±.969\\
    &Diffwave &0.014±.002 &0.232±.032 &0.873±.061 &0.393±.041 &1.031±.131 &0.244±.018\\
    &DiffTime &\underline{0.006±.001} &0.236±.074 &0.299±.044 &0.188±.028 &0.279±.045 &0.340±.015\\
    &Cot-GAN &1.337±.068 &0.408±.086 &0.980±.071 &1.094±.079 &1.039±.028 &7.813±.550\\
    \hline
    \end{tabular}%
    }
\label{tab:experiments}
\end{table*}

%% file: table/long.tex
\begin{table*}[t]
\footnotesize
\relsize{-0.5}
\renewcommand{\arraystretch}{1.3}
\centering
\caption{Results of long-term time series generation, the performances are achieved based on 5 runs.}
\resizebox{1.0\linewidth}{!}{
    \begin{tabular}{c|c|ccc|ccc}
    \hline
    &&&ETTh&&&Energy\\
    \hline
    Metrics &Methods &64 &128 &256&64 &128 &256\\
    \hline
    \multirow{7}{*}{\parbox{1.5cm}{\centering Discriminative\\Score$\downarrow$}} 
    &MSDformer &\textbf{0.013±.012} &\textbf{0.004±.003} &\textbf{0.007±.005} &\textbf{0.009±.006} &\textbf{0.007±.006} &\textbf{0.010±.009}\\
    &SDformer&\underline{0.018±.007} &\underline{0.013±.005} &\underline{0.008±.006} &\underline{0.010±.007} &\underline{0.013±.007} &\underline{0.017±.003} \\
    &Diffusion-TS &0.106±.048 &0.144±.060 &0.060±.030 &0.078±.021 &0.143±.075 &0.290±.123 \\
    &TimeGAN &0.227±.078 &0.188±.074 &0.442±.056 &0.498±.001 &0.499±.001 &0.499±.000 \\
    &TimeVAE &0.171±.142 &0.154±.087 &0.178±.076 &0.499±.000 &0.499±.000 &0.499±.000\\
    &Diffwave &0.254±.074 &0.274±.047 &0.304±.068 &0.497±.004 &0.499±.001 &0.499±.000\\
    &DiffTime &0.150±.003 &0.176±.015 &0.243±.005 &0.328±.031 &0.396±.024 &0.437±.095\\
    &Cot-GAN &0.296±.348 &0.451±.080 &0.461±.010 &0.499±.001 &0.499±.001 &0.498±.004\\
    \hline
    \multirow{8}{*}{\parbox{1.5cm}{\centering Predictive\\Score$\downarrow$}} 
    &MSDformer &\textbf{0.112±.008} &\textbf{0.110±.009} &0.110±.005 &\textbf{0.247±.001} &\textbf{0.244±.001} &\textbf{0.243±.003} \\
    &SDformer&\underline{0.116±.006} &\textbf{0.110±.007} &\textbf{0.095±.003} &\textbf{0.247±.001} &\textbf{0.244±.000} &\textbf{0.243±.002}\\
    &Diffusion-TS &\underline{0.116±.000} &\textbf{0.110±.003} &\underline{0.109±.013} &\underline{0.249±.000} &\underline{0.247±.001} &\underline{0.245±.001} \\
    &TimeGAN &0.132±.008 &0.153±.014 &0.220±.008 &0.291±.003 &0.303±.002 &0.351±.004\\
    &TimeVAE &{0.118±.004} &\underline{0.113±.005} &0.110±.027 &0.302±.001 &0.318±.000 &0.353±.003\\
    &Diffwave &0.133±.008 &0.129±.003 &0.132±.001 &0.252±.001 &0.252±.000 &0.251±.000\\
    &DiffTime &{0.118±.004} &0.120±.008 &0.118±.003 &0.252±.000 &0.251.±.000 &0.251±.000\\
    &Cot-GAN &0.135±.003 &0.126±.001 &0.129±.000 &0.262±.002 &0.269±.002 &0.275±.004\\
    \cline{2-8}
    &Original &0.114±.006 &0.108±.005 &0.106±.010 &0.245±.002 &0.243±.000 &0.243±.000\\
    \hline

    \multirow{7}{*}{\parbox{1.5cm}{\centering Context-FID\\Score$\downarrow$}} 
    &MSDformer &\textbf{0.015±.002} &\textbf{0.023±.002} &\textbf{0.018±.002} &\textbf{0.027±.003} &\textbf{0.034±.004} &\textbf{0.032±.002}\\
    &SDformer&\underline{0.018±.003} &\underline{0.024±.001} &\underline{0.021±.001} &\underline{0.031±.002} &\underline{0.036±.002} &\underline{0.041±.003}\\
    &Diffusion-TS &0.631±.058 &0.787±.062 &0.423±.038 &{0.135±.017} &{0.087±.019} &{0.126±.024}\\
    &TimeGAN &1.130±.102 &1.553±.169 &5.872±.208 &1.230±.070 &2.535±.372 &5.032±.831 \\
    &TimeVAE &0.827±.146 &1.062±.134 &0.826±.093 &2.662±.087 &3.125±.106 &3.768±.998\\
    &Diffwave &1.543±.153 &2.354±.170 &2.899±.289 &2.697±.418 &5.552±.528 &5.572±.584\\
    &DiffTime &1.279±.083 &2.554±.318 &3.524±.830 &0.762±.157 &1.344±.131 &4.735±.729\\
    &Cot-GAN &3.008±.277 &2.639±.427 &4.075±.894 &1.824±.144 &1.822±.271 &2.533±.467\\
    \hline
    \end{tabular}
}
\label{tab:long}
\end{table*}

%% file: table/time.tex
\begin{table}[t]
\centering
\caption{Comparison of inference time (per 1024 samples)  across methods and datasets.}
\renewcommand\arraystretch{1.3}
\begin{tabular}{c|ccccccc}
\hline
Methods &Sines &Stocks &ETTh &MuJoCo &Energy &fMRI  \\ 
\hline
Diffusion-TS &9.57	&10.12	&11.85	&19.34	&34.92	&45.90\\
\rowcolor{gray!25}
SDformer &0.15	&0.18	&0.40	&0.15	&0.15	&0.44\\
\rowcolor{gray!25}
MSDformer &0.95	&0.96 &0.30	&0.97	&0.96	&1.40\\
\hline
\end{tabular}
\label{tab:time}
\end{table}

%% file: table/codebook_expansion.tex
\begin{table}[t]
\centering
\caption{Performance comparison of increasing the single-scale codebook size vs. number of scales under varying length settings. }
\renewcommand\arraystretch{1.3}
\resizebox{1.0\linewidth}{!}{
\begin{tabular}{c|c|ccc}

\hline
$\tau$&$V^{(1:K)}$
&\begin{tabular}[c]{@{}c@{}}Discriminative \\ Score$\downarrow$\end{tabular} &\begin{tabular}[c]{@{}c@{}}Predictive \\ Score$\downarrow$\end{tabular} &\begin{tabular}[c]{@{}c@{}}Context-FID \\ Score$\downarrow$\end{tabular}\\ 
\hline 
\multirow{3}{*}{\parbox{0.5cm}{\centering $24$}} 
&$[512]$ &0.017±.007 &0.091±.002 &0.015±.001\\
&$[640]$ &0.017±.008 &0.091±.003 &0.013±.000\\
&$[128,512]$ &\bf0.005±.003 &\bf 0.090±.002 &\bf 0.009±.001\\
\hline
\multirow{3}{*}{\parbox{0.5cm}{\centering $64$}} 
&$[512]$ &0.013±.011 &\bf 0.083±.003 &0.046±.003\\
&$[640]$ &0.011±.005 &0.085±.003 &0.044±.003\\
&$[128,512]$ &\bf 0.009±.005 &\bf0.083±.002 &\bf 0.041±.003 \\
\hline
\end{tabular}}
\label{tab:codebook_expansion}
\end{table}


%% file: table/scale.tex
\begin{table*}[t]
\centering
\caption{Ablation Study on the Impact of Multi-Scale Information and Scale Sizes. }
\renewcommand\arraystretch{1.3}
\resizebox{1.0\linewidth}{!}{
\begin{tabular}{c|c|ccc|ccc}
\hline
    &&&fRMI&&&Stocks\\
\hline
&$r^{(1:K)}$
&\begin{tabular}[c]{@{}c@{}}Discriminative \\ Score$\downarrow$\end{tabular} &\begin{tabular}[c]{@{}c@{}}Predictive \\ Score$\downarrow$\end{tabular} &\begin{tabular}[c]{@{}c@{}}Context-FID \\ Score$\downarrow$\end{tabular} &\begin{tabular}[c]{@{}c@{}}Discriminative \\ Score$\downarrow$\end{tabular} &\begin{tabular}[c]{@{}c@{}}Predictive \\ Score$\downarrow$\end{tabular} &\begin{tabular}[c]{@{}c@{}}Context-FID \\ Score$\downarrow$\end{tabular}\\ 
\hline
\multirow{3}{*}{\parbox{1.5cm}{\centering $K=1$}} 
&$[2]$ &0.017±.007 &0.091±.002 &0.015±.001 & 0.013±.006 & 0.037±.000 & 0.003±.001\\
&$[4]$ &0.018±.007 &0.092±.002 &0.011±.001 & 0.010±.006 & 0.037±.000 & \textbf{0.002±.000}\\
&$[8]$ &0.032±.012 &0.093±.002 &0.092±.008 & 0.017±.007 & 0.037±.000 & 0.004±.001\\
\hline
\multirow{3}{*}{\parbox{1.5cm}{\centering $K=2$}} 
&$[2,4]$ &\textbf{0.005±.003} &0.091±.002 &\textbf{0.009±.001} & 0.005±.004 & 0.037±.000 & \textbf{0.002±.000}\\
&$[2,8]$ &\textbf{0.005±.003} &0.090±.002 &\textbf{0.009±.000}  & 0.007±.004 & 0.037±.000 & \textbf{0.002±.000}\\
&$[4,8]$ &0.010±.007 &0.092±.002 &0.010±.000  &\textbf{0.004±.004} & 0.037±.000 & 0.003±.000\\
\hline
\multirow{1}{*}{\parbox{1.5cm}{\centering $K=3$}} 
&$[2,4,8]$ &\textbf{0.005±.004} &\textbf{0.089±.003} &\textbf{0.009±.001} & 0.012±.007 & 0.037±.001 & \textbf{0.002±.000}\\
\hline
\end{tabular}}
\label{tab:scale}
\end{table*}

%% file: table/ablation.tex
\begin{table}[t]
\centering
\caption{Results of ablation study with the different modeling strategies of vector quantization in stage 1.}
\renewcommand\arraystretch{1.3}
\tabcolsep=0.12cm
\begin{tabular}{c|cccc}
\toprule 
Methods&\begin{tabular}[c]{@{}c@{}}Discriminative \\ Score$\downarrow$\end{tabular} &\begin{tabular}[c]{@{}c@{}}Predictive \\ Score$\downarrow$\end{tabular} &\begin{tabular}[c]{@{}c@{}}Context-FID \\ Score$\downarrow$\end{tabular}
&\begin{tabular}[c]{@{}c@{}}Codebook \\ Usage$\downarrow$\end{tabular}\\ 
\midrule
 MSDformer &0.005±.004 &0.089±.003 &0.009±.001 &100\%\\
w/o similarity &0.006±.004 &0.092±.002 &0.019±.001 &100\%\\
w/o EMA &0.007±.004 &0.091±.001 &0.011±.000 &100\%\\
w/o Reset &0.423±.099 &0.103±.001 &14.664±1.413 &1.1\%\\
w/o TTE&0.007±.006 &0.090±.002 &0.011±.001 &100\%\\
\bottomrule 
\end{tabular}
\label{tab:ablation}
\end{table}

%% file: figure/vis_ms.tex
\begin{figure}[t!]
\begin{center}
\includegraphics[width=1\linewidth]{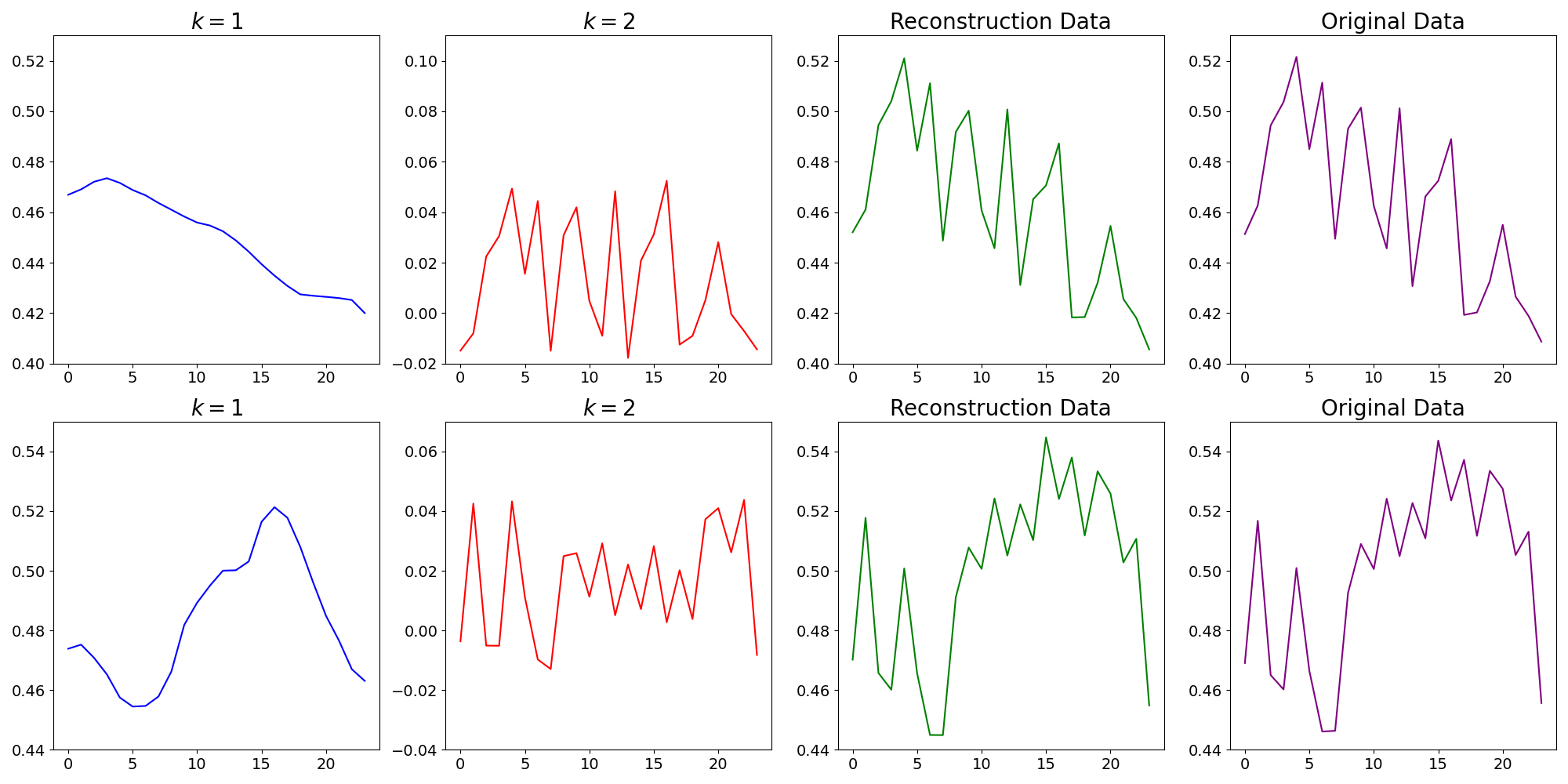}
\end{center}
\caption{Multi-scale time-series visualization for ($K=2$). The first two columns show hierarchical reconstructions at the coarse (($K=1$)) and fine ($K=2$) scales, respectively. The last two columns present the final reconstruction and the original series for comparison.}
\label{fig:vis_ms}
\end{figure}


%% file: supp.tex
\setcounter{section}{0}
\section*{\centering \Large \textbf{Supplementary Materials for
MSDformer}}
In the supplementary, we provide more implementation details of our SDformer and MSDformer. We organize our supplementary as follows:
\begin{itemize}
\item In Appendix \ref{appendix:dataset}, we give the detailed description of used datasets and the metrics.
\item In Appendix \ref{appendix:hyperpar}, we provide the experiment settings.
\item In Appendix \ref{appendix:baselines}, we briefly introduce the baselines.
\item In Appendix \ref{appendix:ab}, we show more experimental results to verify the effect of increasing codebook size vs. number of scales.
\item In Appendix \ref{appendix:model_details}, we provide the details operations of the architecture in stage 1 and 2.
\item In Appendix \ref{appendix:alg}, we  provide the algorithms of training in SDformer and MSDformer.
\end{itemize}
\section{Dataset and Metric Details}
\label{appendix:dataset}
\noindent\textbf{Dataset.} We conduct extensive evaluations of the proposed MSD framework on six widely-used benchmark datasets for multivariate time-series generation task. For real-world datasets, 1) \textbf{Stocks}\footnote{\url{https://finance.yahoo.com/quote/GOOG/history/?p=GOOG}} dataset consists of Google's stock price data between 2004 and 2019, with each observation representing a day and containing 6 features, including as features the volume and high, low, opening, closing, and adjusted closing price. 2) \textbf{ETTh}\footnote{\url{https://github.com/zhouhaoyi/ETDataset}}~\cite{zhou2021informer} dataset includes data obtained from electrical transformers, encompassing load and oil temperature measurements taken every 15 minutes from July 2016 to July 2018. 3) \textbf{Energy}\footnote{\url{https://archive.ics.uci.edu/dataset/374/appliances+energy+prediction}} dataset is a multivariate appliance energy prediction dataset~\cite{candanedo2017data}, featuring 28 channels with correlated features. It exhibits noisy periodicity and contains average 10-minute continuous-valued measurements typical of real-world energy consumption data. 4) \textbf{fMRI}\footnote{\url{https://www.fmrib.ox.ac.uk/datasets/netsim/}} dataset serves as a benchmark for causal discovery and features realistic simulations of blood-oxygen-level-dependent (BOLD) time series; we chose a simulation with 50 features from the original dataset referring to \cite{yuan2024diffusion}. 

As for two simulated datasets, 5) \textbf{Sines}\footnote{\url{https://github.com/jsyoon0823/TimeGAN}} is to simulate multivariate sinusoidal sequences of different frequencies $\eta$ and phases $\theta$ providing continuous-valued, periodic, multivariate data where each feature is independent of others. For each dimension $i \in \{1,...,5\}, x_{i}(t)= \text{sin}(2\pi\eta t+ \theta)$, where $\eta \sim \mathcal{U}[0, 1]$ and $\theta \sim \mathcal{U}[-\pi, \pi]$. 6) \textbf{MuJoCo}\footnote{\url{https://github.com/google-deepmind/dm_control}}~\cite{todorov2012mujoco, rubanova2019latent} dataset is a multivariate physics simulation time series dataset, which  collects a total of 10,000 simulations of the ``Hopper'' model from the DeepMind Control Suite and MuJoCo simulator. The position of the body in 2D space is uniformly sampled from the interval $[0,0.5]$. The relative position of the limbs is sampled from the range $[-2,2]$, and initial velocities are sampled from the interval $[-5,5]$. In all, there are 10000 sequences of 100 regularly sampled time points with a feature dimension of 14.

\noindent\textbf{Quantitative Metrics.} To be specific, 

1). \textbf{Discriminative Score} measures the distributional similarity between original and synthesized time series data. A binary classifier (e.g., RNNs-based) is trained to distinguish original data (labeled as ``real'') from synthesized data (labeled as ``synthetic''), followed by evaluating its classification accuracy on a test set. The score is derived as the absolute difference between the classifier's accuracy and the random-guessing baseline (0.5), where a value closer to 0 indicates higher similarity between the two datasets. For a fair comparison, we follow the standard TimeGAN and Diffusion-ts~\cite{yoon2019time, yuan2024diffusion} style post-hoc evaluation protocol: we train an off-the-shelf RNN-based binary classifier to discriminate real sequences from generated sequences on a held-out test set, and compute the discriminative score as $|\text{Acc}-0.5|$ (lower is better). And the classifier is a GRU-based sequence classifier consisting of a GRU structure followed by a fully-connected sigmoid head. Specifically, given the input sequences $x\in\mathbb{R}^{T\times d}$ (Ground truths and generation results), where $T$ is the generation length and $d$ is the dimension of dataset, we run a GRU model and map it to a scalar logit via a linear layer; the probability is obtained by a sigmoid. Also, within the training setup, we construct a balanced dataset by labeling each original sequence as real (1) and each generated sequence as fake (0). The classifier is trained using binary cross-entropy with the Adam optimizer for 2000 iterations and batch size 128, where each mini-batch contains samples from both real and generated training sets, which follows the experimental settings of Diffusion-ts\footnote{\url{https://github.com/Y-debug-sys/Diffusion-TS/tree/main}}. We evaluate accuracy on the held-out test split and report the discriminative score.

2). \textbf{Predictive Score} evaluates the temporal fidelity of synthesized data by training a post-hoc one-step-ahead sequence predictor (e.g., RNNs-based) on synthetic data to predict future time steps and testing its generalization on the original dataset. The score is quantified via the model’s mean absolute error (MAE) on real test samples, where lower values indicate better preservation of the original data’s temporal dynamics and we also follow the previous works's experimental settings for fair evaluation. Concretely, given a input sequence $x_{1:T}\in\mathbb{R}^{T\times d}$, the predictor takes the prefix $x_{1:T-1}$ and predicts the next-step sequence $x_{2:T}$. In our implementation, the input context length is (T-1) time steps (all but the last step of each window), and the prediction target is also (T-1) steps, corresponding to a one-step-shifted forecasting objective. Thus, the predictive score does not evaluate long-horizon forecasting; it evaluates one-step-ahead prediction across the entire window following TimeGAN’s post-hoc predictive score protocol. The predictor is trained with L1 loss (MAE) using Adam for 5000 iterations with batch size 128 and the structure is aslo a GRU-based RNN model. Each training iteration samples a mini-batch of synthetic sequences, constructs input–target pairs ($x_{1:T-1}, x_{2:T}$), and performs one optimizer step. After training, we evaluate the trained predictor on real sequences and average MAE over all sequences to obtain the predictive score. Overall, it trains a one-step-ahead predictor (shift-by-one) on synthetic data and reports the MAE on real data over all time steps in the generation window.

3). \textbf{Context-FID Score} evaluates the distributional fidelity of synthesized time series by comparing contextual feature embeddings between original and synthetic datasets. A self-supervised embedding network (e.g., trained via temporal contrastive learning) first extracts context-aware representations from both datasets. The Fréchet distance is then computed between the multivariate Gaussian distributions of these embeddings, characterized by their means and covariance matrices. And the Context-FID Score is to replace the Inception model of the original Frechet Inception distance (FID)~\cite{jeha2022psa} with a 
time series representation learning method called TS2Vec~\cite{yue2022ts2vec}. Lower scores indicate closer alignment in contextual patterns between real and synthetic data. Specifically, we first sample synthetic time series and real-time series respectively. Then we compute the FID score of the representation after encoding them with a pre-trained TS2Vec model across each dataset, and each embedding is computed from the entire generation sequence, not from a randomly selected sub-range. Details are shown in Fig.~\ref{fig:ts2vec-dataflow-simple}.

\noindent\textbf{Qualitative Metric}. The \textbf{t-SNE} \cite{van2008visualizing} method projects high-dimensional time series data into 2D space to reveal underlying structural patterns. It first computes pairwise similarities in the original space using adaptive Gaussian kernels (guided by perplexity parameters), then optimizes a low-dimensional t-distribution representation through KL divergence minimization. The final visualization preserves neighborhood relationships while mitigating crowding effects, enabling intuitive comparison of real and synthetic data distributions.
The \textbf{Kernel Density Estimation} (KDE) evaluates distribution alignment by estimating probability densities through Gaussian kernel smoothing across multiple dimensions. Using bandwidth-optimized kernels, it generates continuous density profiles for both datasets. Visual comparison of these smoothed distributions reveals preservation of key statistical patterns in synthetic data, with overlapping contours indicating higher fidelity.

\begin{figure}[h]
    \centering
    \setlength{\tabcolsep}{6pt}
    \renewcommand{\arraystretch}{1.25}
    \begin{tabular}{>{\raggedright\arraybackslash}p{0.93\linewidth}}
    \hline
    \textbf{Input:} $x\in\mathbb{R}^{B\times T\times D}$ \\[2pt]
    $\Downarrow$ \\
    \textbf{NaN mask:} $m_{\text{nan}}\in\{0,1\}^{B\times T}$, set $x[\neg m_{\text{nan}}]\leftarrow 0$ \\[2pt]
    $\Downarrow$ \\
    \textbf{Linear projection:} $\mathrm{Linear}(D\!\rightarrow\!H)$, $h^{(0)}\in\mathbb{R}^{B\times T\times H}$ \\[2pt]
    $\Downarrow$ \\
    \textbf{Time masking (train):} sample $m\in\{0,1\}^{B\times T}$, apply $m\leftarrow m\wedge m_{\text{nan}}$, set $h^{(0)}[\neg m]\leftarrow 0$ \\[2pt]
    $\Downarrow$ \\
    \textbf{Transpose:} $B\times T\times H \rightarrow B\times H\times T$ \\[2pt]
    $\Downarrow$ \\
    \textbf{Dilated Conv Encoder:} residual ConvBlocks with dilation $2^i$ and SamePadConv (length preserved), output $\hat h^{(L)}\in\mathbb{R}^{B\times C_{\text{out}}\times T}$ \\[2pt]
    $\Downarrow$ \\
    \textbf{Dropout:} $p=0.1$ \\[2pt]
    $\Downarrow$ \\
    \textbf{Transpose back:} $B\times C_{\text{out}}\times T \rightarrow B\times T\times C_{\text{out}}$ \\[2pt]
    $\Downarrow$ \\
    \textbf{Output:} $z\in\mathbb{R}^{B\times T\times C_{\text{out}}}$, where $C_{\text{out}}=\text{output\_dims}$ \\
    \hline
    \end{tabular}
    \caption{Details of the TS2Vec encoder. }
    \label{fig:ts2vec-dataflow-simple}
\end{figure}
\section{Experimental Settings}
\label{appendix:hyperpar}
In this part, we introduce our main experimental settings. For the time series generation task, we conduct five evaluations to obtain the experimental results. Furthermore, the detailed hyperparameters of SDformer and MSDformer are summarized in Tables~\ref{tab:hyperparameters_sdformer} and~\ref{tab:hyperparameters_msdformer}, respectively. Parameter counts for both models are provided in Table~\ref{tab:modelsize}. 
\input{table/parameters_sdformer}
\input{table/parameters_msdformer}
\input{table/modelsize}
\input{table/codebook_expansion2}

\input{table/transformer}
\vspace{-0.2in}
\input{table/encoder_1}
\input{table/decoder_1}
\input{table/Resnet}
\input{table/stage2_train}
\section{Baselines}
\label{appendix:baselines}
In this part, we briefly introduce the baselines:

\textbf{TimeGAN} \cite{yoon2019time} is a method that combines Generative Adversarial Networks (GANs) with sequence modeling. It jointly optimizes supervised and unsupervised adversarial losses to learn both the static features and dynamic dependencies of time series data, thereby generating high-quality synthetic time series.

\textbf{COT-GAN} \cite{xu2020cot} is a GAN-based method that incorporates causal optimal transport theory. It optimizes the distribution matching between generated and real data to produce high-quality synthetic time series.

\textbf{TimeVAE} \cite{desai2021timevae} is a time series generation method based on Variational Autoencoders (VAEs). It enhances model interpretability by incorporating domain-specific temporal patterns (such as polynomial trends and seasonality) and optimizes training efficiency to accurately represent and efficiently synthesize the temporal attributes of the original data.

\textbf{DiffTime} \cite{coletta2024constrained} is a time series generation method based on a constrained optimization framework. It combines conditional diffusion models to efficiently generate synthetic time series that meet specific constraints.

\textbf{DiffWave} \cite{kong2020diffwave} is an audio synthesis method based on diffusion probabilistic models. The model combines a non-autoregressive structure with bidirectional dilated convolutions to generate high-quality audio in parallel, offering significant advantages in synthesis speed. It also ensures the quality of the generated audio by progressively denoising white noise to produce the target audio waveform.

\textbf{Diffusion-TS} \cite{yuan2024diffusion} is a method that uses an encoder-decoder Transformer architecture to decouple the seasonality and trend of time series. It leverages diffusion models for high-precision multivariate time series generation. Additionally, the algorithm directly reconstructs samples rather than noise and uses Fourier loss to constrain periodic characteristics, significantly improving the quality and interpretability of the generated time series.

\section{Additional experimental results}
\label{appendix:ab}
In this part, we additionally provide an experimental performance comparison between increasing codebook size and number of scales under different length settings on the Energy dataset, as shown in Table \ref{tab:codebook_expansion2}.
The results align with the conclusions drawn in the main text: using SDformer with a codebook size of \( V = 512 \) as the baseline, increasing the codebook size to \( V = 1024 \) leads to performance improvements. However, the performance enhancement is more pronounced and stable when increasing the number of scales (e.g., \( V^{(1:K)} = [512, 512] \), i.e., MSDformer). This further demonstrates the advantages of multi-scale modeling.

\section{Algorithms}
\label{appendix:alg}
In this section, we summarize the complete training and inference procedures of MSDformer, which follows a two-stage pipeline and are shown in the main manuscript. Stage 1 trains a multi-scale time-series tokenizer to discretize continuous sequences into compact code indices while preserving reconstruction fidelity, providing the discrete representation on which the rest of the framework is built. Stage 2 trains a multi-scale autoregressive Transformer to model the joint distribution of the resulting multi-scale token sequences, which enables unconditional sampling of new tokens and, through the tokenizer's decoder, the generation of novel time-series samples. The detailed optimization steps for Stage 1 and Stage 2 are summarized in Algorithms \ref{alg:vqvae} and \ref{alg:train_ms_ar}, respectively, while the overall inference pipeline is given in Algorithm \ref{alg:msdformer}. SDformer is a special case of this framework obtained by setting the number of scales to ($K$=1).

\section{Model details} 
\label{appendix:model_details}
In this section, we present the detailed network architecture of SDformer and MSDformer.The meanings of the symbols in following tables are as follows: B is batch size, $d$ is the input channels, D and $d_{c}$ denote the hidden and codebook dimensions, respectively. L, T are the sequence length and number of downsampling/upsampling operations in our setting. Also the i in Table \ref{tab:enc_1} and \ref{tab:dec_1} represents iteration index (0 to T-1). In Stage 1, MSDformer and SDformer employ identical encoder architectures and identical decoder architectures. The only structural variation arises from differing layer counts in specific modules due to distinct downsampling factors, as documented in Tables \ref{tab:enc_1} and \ref{tab:dec_1}. Also, the ResNet block of Stage 1 is described in Table \ref{tab:resnet_1}.
In stage 2, the multi-scale autoregressive Discrete
Transformer is implemented as a standard decoder-only Transformer, with its block structure depicted in Table \ref{table:trans_ar} and the detailed training process is shown in Table \ref{tab:stage2_train}, in which the $\text{P}_{\text{sum}}$ is sum of the token length within each scale and $\text{C}_{\text{sum}}$ denotes the total codebook size of the multi-scale tokenizers. Since the T denotes the number of upsampling and downsampling operations. The total layers of Encoder include 2 initial layers (Conv + ReLU), 2*T repeated layers (Conv + ResNet) and 1 final layer (Conv). For the decoder structure, it consists of 2 initial layers (Conv + ReLU), 3*T repeated layers (ResNet + Upsample + Conv) and 3 final layers (Conv + ReLU + Conv).

Finally, the multi-scale codebook sizes $V^{(1:K)}$ used by MSDformer are selected empirically per dataset. Since different datasets exhibit different dimensionality, which requires different effective token rate budgets for a favorable rate–distortion trade-off. As future work, we plan to explore the relationship between the optimal token budget (e.g., effective rate ($\sum_k L^{(k)}\log_2 V^{(k)})$) and dataset dimensionality, with the goal of deriving an automatic and theoretically grounded configuration strategy for $V^{(1:K)}$. 

\clearpage
\input{figure/vis_tsne}
\vspace{-0.4mm}
\input{figure/vis_kernel}

%% file: table/parameters_sdformer.tex
\begin{table}[h!]
    \caption{Detailed hyperparameters of SDformer.}
    \centering
    \renewcommand\arraystretch{1.3}
    \begin{threeparttable}
        \begin{tabular}{c|cccccc}
                \toprule
                Dataset  & Sines & Stocks & ETTh & MuJoCo & Energy & fMRI \\
                \midrule

                $V$             & 1024 & 512 & 512 & 512 & 512 & 512 \\
                $d_c$           & 512 & 256 & 512 & 512 & 512 & 512 \\
                $\lambda$       & 0.5 & 2.0 & 0.5 & 0.5 & 0.001 & 0.01 \\
                $r$             & 4 & 4 & 4 & 4 & 4 & 2 \\
                $N_{enc}$  & 2 & 2 & 2 & 2 & 2 & 1 \\
                $N_{trans}$  & 2 & 2 & 6 & 2 & 2 & 2 \\
                $\epsilon$               & 0.3 & 0.3 & 0.3 & 0.1 & 0.1 & 0.1 \\
                \bottomrule
        \end{tabular}
        \begin{tablenotes}
            \small
            \item[$N_{enc}$:] Number of encoder layers
            \item[$N_{trans}$:] Number of transformer layers
        \end{tablenotes}
    \end{threeparttable}
    \label{tab:hyperparameters_sdformer}
\end{table}

%% file: table/parameters_msdformer.tex
\begin{table}[h!]
    \caption{Detailed hyperparameters of MSDformer.}
    \centering
    \tabcolsep=0.02cm
    \renewcommand\arraystretch{1.3}
    \begin{tabular}{c|cccccc}
                \toprule
                Dataset  & Sines & Stocks & ETTh & MuJoCo & Energy & fMRI \\
                \midrule

                $V^{(1:K)}$             & [512,512] & [256,512] &[128,512] & [512,512] &[512,512] & [128,512,512]\\
                $d_c$ &512 &512 &512 &512 &512 &512\\
                $\lambda$       & 1.0 & 0.5 & 0.5 &1.0 & 0.0005 & 0.01 \\
                $r^{(1:K)}$             &[2,4] & [2,4] & [4,8] & [2,4] & [2,4] & [2,4,8] \\
                $N_{enc}$  & [1,2] & [1,2] & [2,3] & [1,2] & [1,2] & [1,2,3] \\
                $N_{trans}$ & 2 & 2 & 2 & 2 & 2 & 2\\
                 $\epsilon$               & 0.1 & 0.3 & 0.1 & 0.1 & 0.1 & 0.1 \\
                \bottomrule
    \end{tabular}
    \label{tab:hyperparameters_msdformer}
\end{table}

%% file: table/modelsize.tex
\begin{table}[t]
\footnotesize
\centering
\caption{Parameter counts of SDformer and MSDformer on various Datasets (Unit: Million). Total parameters are calculated from precise values; minor discrepancies in summation are due to rounding.}
\label{tab:modelsize}
\renewcommand{\arraystretch}{1.3}
\resizebox{\linewidth}{!}{%
\begin{tabular}{c|ccc|ccc|ccc}
\hline
\multirow{2}{*}{Methods} & 
\multicolumn{3}{c|}{Sines} & 
\multicolumn{3}{c|}{Stocks} & 
\multicolumn{3}{c}{ETTh}  \\
\cline{2-10}
& Stage1 & Stage2 & Total & Stage1 & Stage2 & Total & Stage1 & Stage2 & Total  \\
\hline
SDformer & 18.6 & 27.3 & 45.9 & 18.6 & 25.7 & 44.4 & 18.7 & 76.6 & 95.3  \\
MSDformer &29.2 &27.3 &56.5 &29.2 &26.8 &55.9 &45.4 &26.5 &72.0 \\
\hline

\hline
\multirow{2}{*}{Methods} & 
\multicolumn{3}{c|}{MuJoCo} & 
\multicolumn{3}{c|}{Energy} & 
\multicolumn{3}{c}{fMRI} \\
\cline{2-10}
& Stage1 & Stage2 & Total & Stage1 & Stage2 & Total & Stage1 & Stage2 & Total  \\
\hline
SDformer & 18.7 & 26.3 & 44.9 & 18.7 & 26.3 & 45.0 & 10.6 & 26.3 & 36.9 \\
MSDformer &29.2 &27.3 &56.5 &29.3 &27.3 &56.6 &56.3 &27.6 &83.9
 \\
\hline
\end{tabular}%
}
\end{table}

%% file: table/codebook_expansion2.tex
\begin{table}[hbtp]
\centering
\caption{Performance comparison of increasing the single-scale codebook size vs. number of scales under varying length settings. }
\renewcommand\arraystretch{1.3}
\resizebox{1.0\linewidth}{!}{
\begin{tabular}{c|c|ccc}

\hline
$\tau$&$V^{(1:K)}$
&\begin{tabular}[c]{@{}c@{}}Discriminative \\ Score$\downarrow$\end{tabular} &\begin{tabular}[c]{@{}c@{}}Predictive \\ Score$\downarrow$\end{tabular} &\begin{tabular}[c]{@{}c@{}}Context-FID \\ Score$\downarrow$\end{tabular}\\ 
\hline
\multirow{3}{*}{\parbox{0.5cm}{\centering $24$}} 
&$[512]$ &0.006±.004 &0.249±.000 &0.003±.000\\
&$[1024]$ &0.005±.004 &0.249±.000 &0.003±.000\\
&$[512,512]$ &0.005±.003 &0.249±.000 &0.003±.000\\
\hline
\multirow{3}{*}{\parbox{0.5cm}{\centering $64$}} 
&$[512]$ &0.010±.007 &0.247±.001 &0.031±.002\\
&$[1024]$ &0.010±.005 &0.247±.001 &0.034±.003\\
&$[512,512]$ &0.009±.006 &0.247±.001 &0.027±.003 \\
\hline
\multirow{3}{*}{\parbox{0.5cm}{\centering $128$}} 
&$[512]$ &0.013±.000 &0.244±.000  &0.036±.002\\
&$[1024]$ &0.010±.008 &0.245±.001 &0.031±.003\\
&$[512,512]$ &0.007±.006  &0.244±.001 &0.034±.004 \\
\hline
\multirow{3}{*}{\parbox{0.5cm}{\centering $256$}} 
&$[512]$ &0.017±.003 &0.243±.002 &0.041±.003\\
&$[1024]$ &0.011±.005 &0.243±.002 &0.034±.003\\
&$[512,512]$ &0.010±.009 &0.243±.003 &0.032±.002 \\
\hline
\end{tabular}}
\label{tab:codebook_expansion2}
\end{table}

%% file: table/transformer.tex
\begin{table}[th]
    \caption{
    The detailed architecture of the Autoregressive Transformer block.
    }
\newcommand{\tabincell}[2]{\begin{tabular}{@{}#1@{}}#2\end{tabular}}
 \begin{center}
 	\begin{tabular}{|c|c|c|}
 	\multicolumn{1}{c}{Layer} & \multicolumn{1}{c}{Function} & \multicolumn{1}{c}{Descriptions} \\
        \toprule
 	 1 & Layernorm & nn.LayerNorm()\\
          2 & Casual-attention & CasualAttention(q=x, k=x, v=x) \\
          3 & Layernorm & nn.LayerNorm() \\
          4 & MLP & nn.Linear() \\
          5 & ReLU & nn.ReLU()\\
          6 & MLP & nn.Linear() \\
    \bottomrule
	\end{tabular}
	 \end{center}
  \label{table:trans_ar}
\end{table}

%% file: table/encoder_1.tex
\begin{table*}[h]
    \caption{
    The detailed architecture of the time serie tokenizer's encoder. The ellipsis indicates that layers 3 to 5 are repeated, with the number of repetitions increasing by one for each doubling of the downsampling factor of the encoder.}
    
\newcommand{\tabincell}[2]{\begin{tabular}{@{}#1@{}}#2\end{tabular}}
 \begin{center}
    \resizebox{1.0\linewidth}{!}{
 	\begin{tabular}{|c|c|c|c|c|}
 	\multicolumn{1}{c}{Layer} & \multicolumn{1}{c}{Function} & \multicolumn{1}{c}{Descriptions} & \multicolumn{1}{c}{Input Shape} & \multicolumn{1}{c}{Output Shape} \\
    \toprule
 	 1 & Convolution & input channel=$d$, output channel=D, kernel size=3, stride=1, padding=1 & (B, $d$, L) & (B, D, L) \\
        2   & ReLU & nn.ReLU() & (B, D, L)  & (B, D, L) \\
       3 & Convolution & input channel=D, output channel=D, kernel size=4, stride=2, padding=1 & (B, D, L) & (B, D, L$/$2) \\
       4 &ResNet &input channel=D, depth=3, dilation growth rate=3 & (B, D, L$/$2) & (B, D, L$/$2) \\
       $\dots$ & $\dots$ & \textcolor{gray}{Repeat layers 3-4 for (T-1) more times} & (B, D, $\text{L}/2^i$) & (B, D, $\text{L}/2^{i+1}$) \\ 
       2T+3 & Convolution & input channel=D, output channel=$d_{c}$, kernel size=3, stride=1, padding=1 & (B, D, $\text{L}/2^{T}$)  & (B, $d_{c}$, $\text{L}/2^{T}$)\\
    \bottomrule
	\end{tabular}}
	 \end{center}
  \label{tab:enc_1}
\end{table*}

%% file: table/decoder_1.tex
\begin{table*}
    \caption{
    The detailed architecture of the time serie tokenizer's decoder. The ellipsis indicates that layers 3 to 5 are repeated, with the number of repetitions increasing by one for each doubling of the downsampling factor of the encoder.}
    
    \label{tab:Tran-Enc}
\newcommand{\tabincell}[2]{\begin{tabular}{@{}#1@{}}#2\end{tabular}}
 \begin{center}
    \resizebox{1.0\linewidth}{!}{
 	\begin{tabular}{|c|c|c|c|c|}
 	\multicolumn{1}{c}{Layer} & \multicolumn{1}{c}{Function} & \multicolumn{1}{c}{Descriptions} & \multicolumn{1}{c}{Input Shape} & \multicolumn{1}{c}{Output Shape} \\
    \toprule
 	 1 & Convolution & input channel=$d_{c}$, output channel=D, kernel size=3, stride=1, padding=1 & (B, $d_{c}$, $\text{L}/2^{T}$) & (B, D, $\text{L}/2^{T}$) \\
     2  & ReLU & nn.ReLU() & (B, D, $\text{L}/2^{T}$)  & (B, D, $\text{L}/2^{T}$) \\
     3 & ResNet & input channel=D, depth=3, dilation growth rate=3 & (B, D, $\text{L}/2^{T-i}$) & (B, D, $\text{L}/2^{T-i}$) \\ 
     4 & Upsample & nn.Upsample(scale factor=2, mode='nearest') & (B, D, $\text{L}/2^{T-i}$) & (B, D, $\text{L}/2^{T-i-1}$) \\
     5 & Convolution & input channel=D, output channel=D, kernel size=3, stride=1, padding=1 & (B, D, $\text{L}/2^{T-i-1}$) & (B, D, $\text{L}/2^{T-i-1}$) \\
     \dots & \dots & \textcolor{gray}{Repeat layers 3-5 for (T-1) more times} & \dots & \dots \\
     3T+3 & Convolution & input channel=D, output channel=D, kernel size=3, stride=1, padding=1 & (B, D, L) & (B, D, L)  \\
     3T+4 & ReLU  & nn.ReLU() & (B, D, L) & (B, D, L) \\
     3T+5 & Convolution  & input channel=D, output channel=$d$, kernel size=3, stride=1, padding=1 & (B, D, L) & (B, $d$, L) \\
    \bottomrule
	\end{tabular}}
	 \end{center}
  \label{tab:dec_1}
\end{table*}

%% file: table/Resnet.tex
\begin{table*}[h]
    \caption{
    The detailed architecture of the ResNet block.}
    
\newcommand{\tabincell}[2]{\begin{tabular}{@{}#1@{}}#2\end{tabular}}
 \begin{center}
    \resizebox{1.0\linewidth}{!}{
 	\begin{tabular}{|c|c|c|c|c|}
 	\multicolumn{1}{c}{Layer} & \multicolumn{1}{c}{Function} & \multicolumn{1}{c}{Descriptions} & \multicolumn{1}{c}{Input Shape} & \multicolumn{1}{c}{Output Shape} \\
    \toprule
 	 1 & LayerNorm & nn.LayerNorm() & (B, D, L) & (B, D, L) \\
     2   & ReLU & nn.ReLU() & (B, D, L)  & (B, D, L) \\
     3 & Convolution & input channel=D, output channel=D, kernel size=3, padding=dilation, dilation=dilation  & (B, D, L) & (B, D, L) \\
     4   & LayerNorm & nn.LayerNorm() & (B, D, L) & (B, D, L) \\
     5 & ReLU & nn.ReLU() & (B, D, L)  & (B, D, L) \\ 
     6 & Convolution & input channel=D, output channel=D, kernel size=1, stride=1, padding=0  & (B, D, L)  & (B, D, L)\\
     7 & Residual & output = input + conv output  & (B, D, L)  & (B, D, L) \\
    \bottomrule
	\end{tabular}}
	 \end{center}
  \label{tab:resnet_1}
\end{table*}

%% file: table/stage2_train.tex
\begin{table*}[h!]
    \caption{
    The detailed training process of stage 2.}
    
\newcommand{\tabincell}[2]{\begin{tabular}{@{}#1@{}}#2\end{tabular}}
 \begin{center}
    \resizebox{1.0\linewidth}{!}{
 	\begin{tabular}{|c|c|c|c|c|}
 	\multicolumn{1}{c}{Process} & \multicolumn{1}{c}{Module} & \multicolumn{1}{c}{Operation} & \multicolumn{1}{c}{Input Shape} & \multicolumn{1}{c}{Output Shape} \\
    \toprule
 	 Input & --- & Codebook indices inputs & (B, $\text{P}_{\text{sum}}$) & --- \\
     Add [BOS] token & prepend & Add start token &  (B, $\text{P}_{\text{sum}}$)  &  (B, 1+$\text{P}_{\text{sum}}$) \\
     Token Emb & tok\_embed  & Embedding lookup   & (B, 1+$\text{P}_{\text{sum}}$) & (B, 1+$\text{P}_{\text{sum}}$, $d_{c}$) \\
     Pos Emb  & pos\_embed  & Add position encoding & (B, 1+$\text{P}_{\text{sum}}$, $d_{c}$) & (B, 1+$\text{P}_{\text{sum}}$, $d_{c}$) \\
     Type Emb & type\_emb  & Add type encoding & (B, 1+$\text{P}_{\text{sum}}$, $d_{c}$)  & (B, 1+$\text{P}_{\text{sum}}$, $d_{c}$) \\ 
     Transformer Blocks & block $\times$ N & Transformer layers  & (B, 1+$\text{P}_{\text{sum}}$, $d_{c}$)  & (B, 1+$\text{P}_{\text{sum}}$, $d_{c}$)\\
     Linear Head & head projection & Project to vocabulary  & (B, 1+$\text{P}_{\text{sum}}$, $d_{c}$)  & (B, 1+$\text{P}_{\text{sum}}$, $\text{C}_{\text{sum}}$) \\
     Outputs & --- & Logits for next token & (B, 1+$\text{P}_{\text{sum}}$, $\text{C}_{\text{sum}}$) & ---\\
    \bottomrule
	\end{tabular}}
	 \end{center}
  \label{tab:stage2_train}
\end{table*}

%% file: figure/vis_tsne.tex
\begin{figure*}
\centering

\raisebox{0.2\height}{\rotatebox{90}{\small MSDformer}} 
\subfloat{\includegraphics[width=0.16\textwidth]{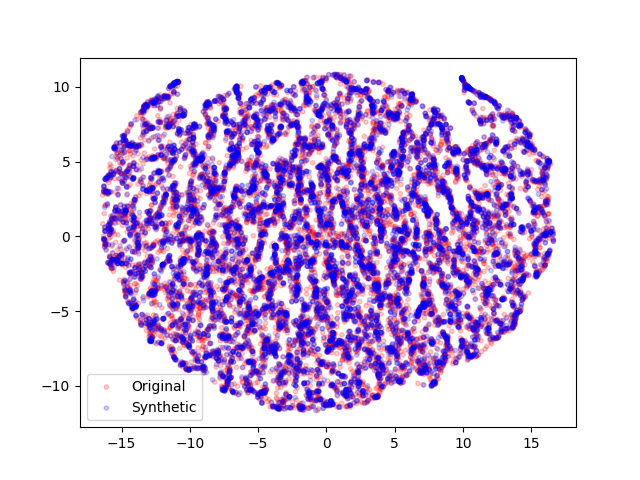}}
\hfill
\subfloat{\includegraphics[width=0.16\textwidth]{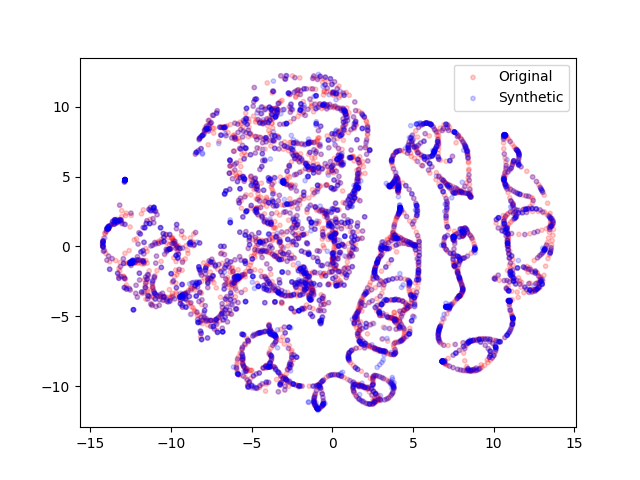}}
\hfill
\subfloat{\includegraphics[width=0.16\textwidth]{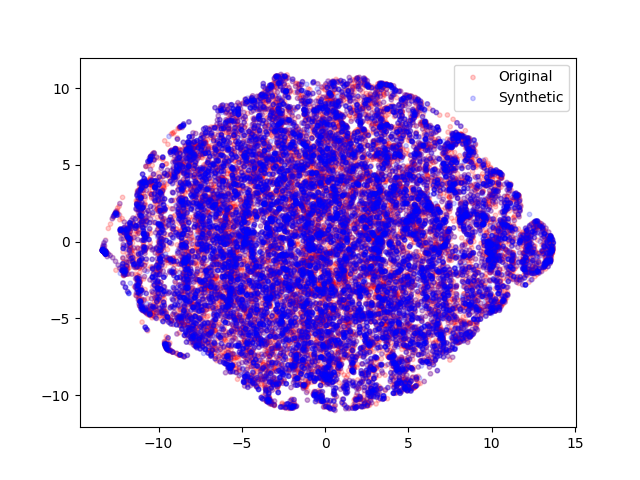}}
\hfill
\subfloat{\includegraphics[width=0.16\textwidth]{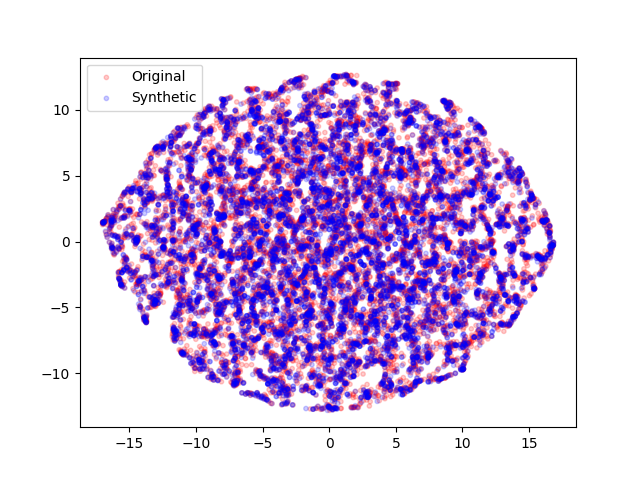}}
\hfill
\subfloat{\includegraphics[width=0.16\textwidth]{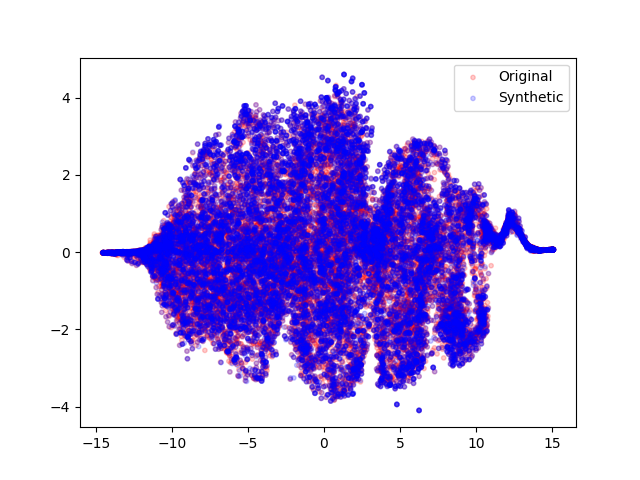}}
\hfill
\subfloat{\includegraphics[width=0.16\textwidth]{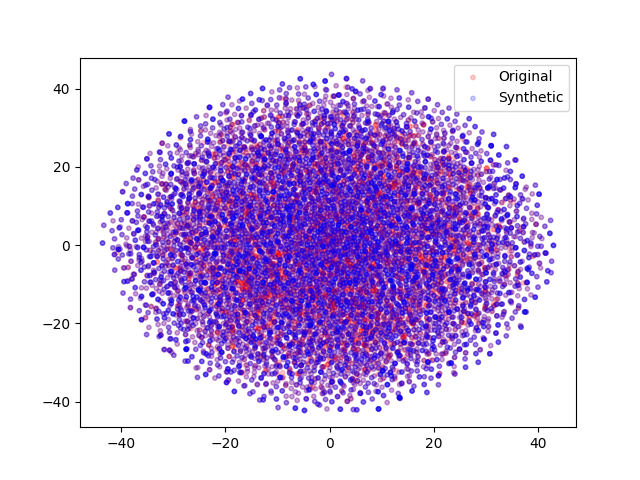}}

\vspace{-4mm}
\par
\raisebox{0.4\height}{\rotatebox{90}{\small SDformer}} 
\subfloat{\includegraphics[width=0.16\textwidth]{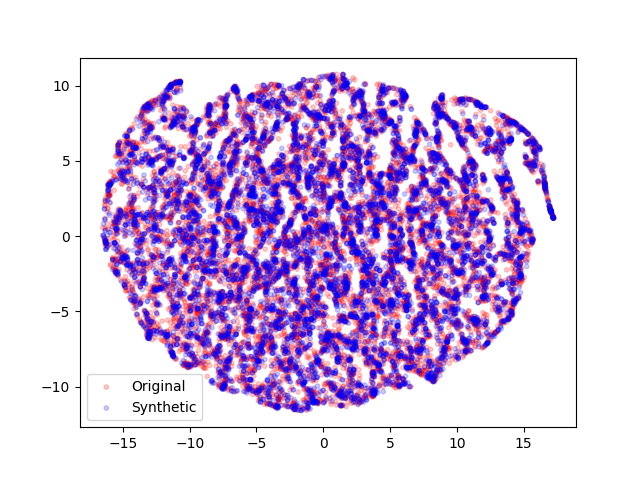}}
\hfill
\subfloat{\includegraphics[width=0.16\textwidth]{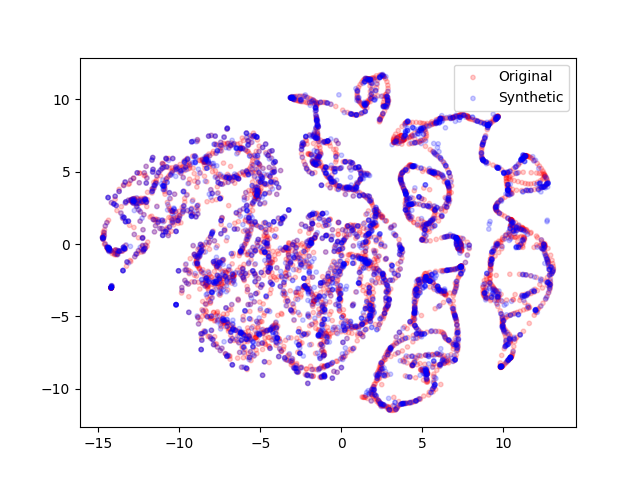}}
\hfill
\subfloat{\includegraphics[width=0.16\textwidth]{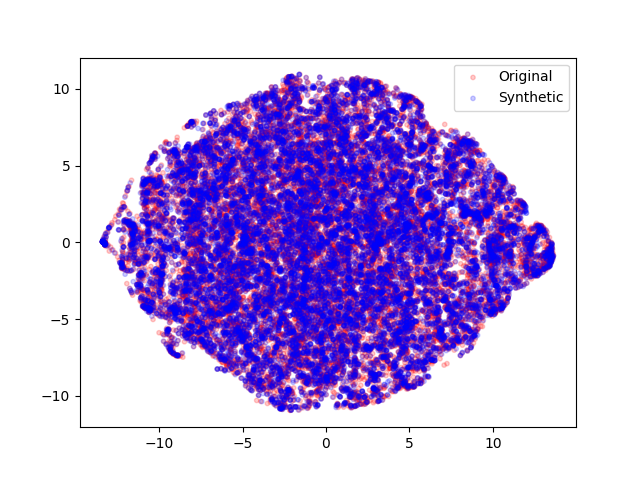}}
\hfill
\subfloat{\includegraphics[width=0.16\textwidth]{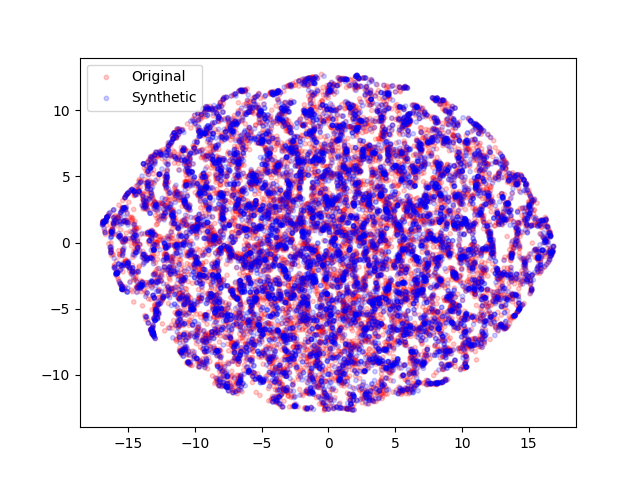}}
\hfill
\subfloat{\includegraphics[width=0.16\textwidth]{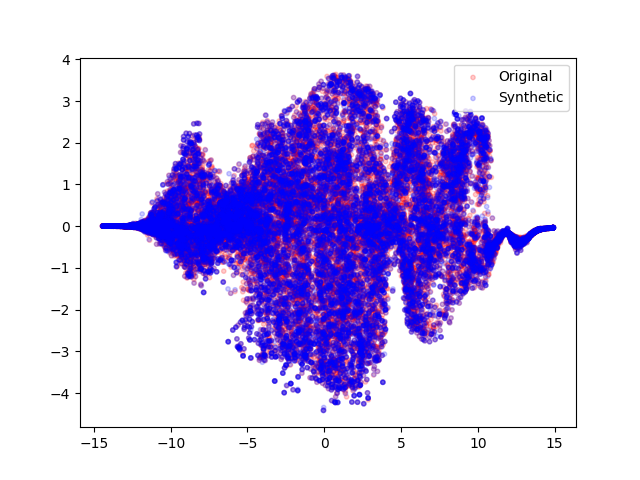}}
\hfill
\subfloat{\includegraphics[width=0.16\textwidth]{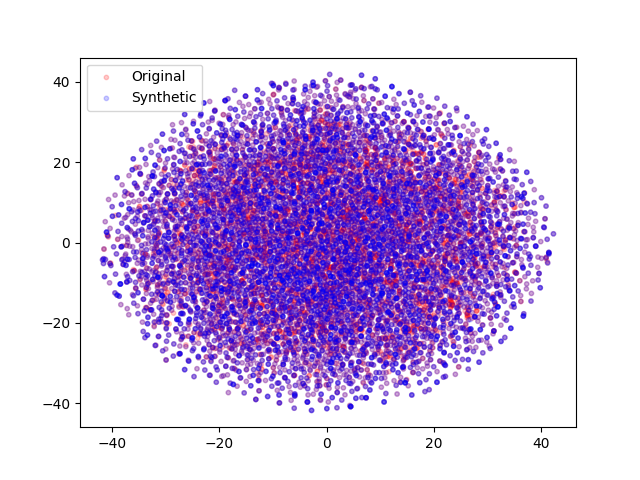}}

\vspace{-4mm}
\par
\setcounter{subfigure}{0}
\raisebox{0.1\height}{\rotatebox{90}{\small Diffusion-TS}} 
\subfloat[Sines]{\includegraphics[width=0.16\textwidth]{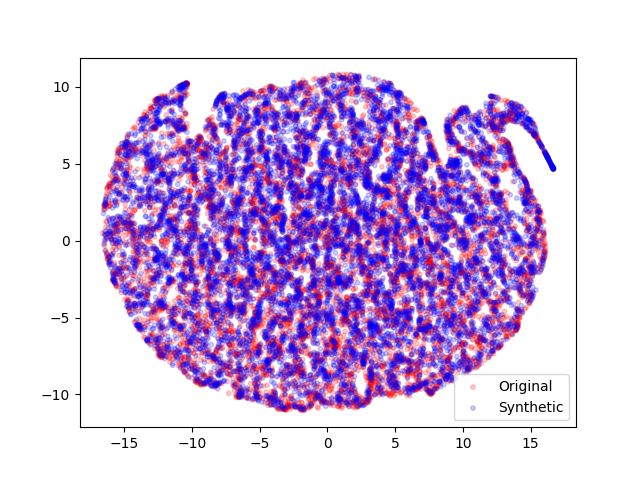}}
\hfill
\subfloat[Stocks]{\includegraphics[width=0.16\textwidth]{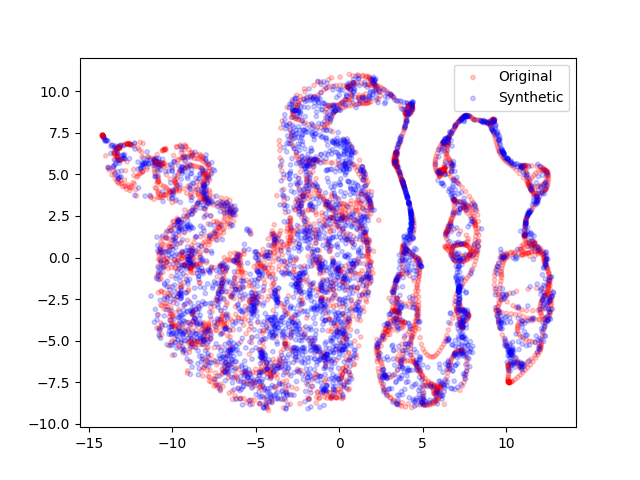}}
\hfill
\subfloat[ETTh]{\includegraphics[width=0.16\textwidth]{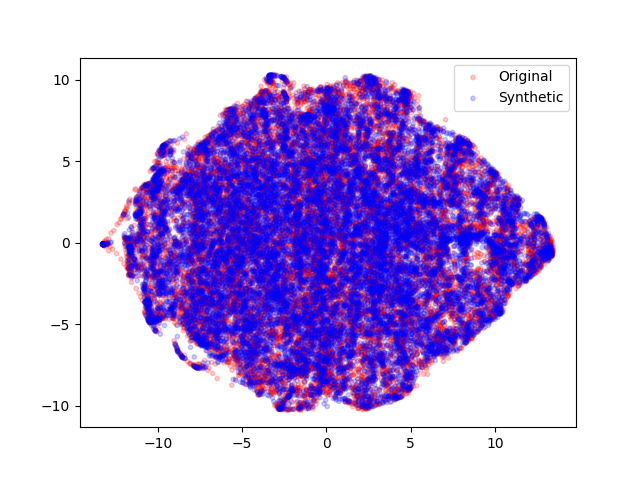}}
\hfill
\subfloat[MuJoCo]{\includegraphics[width=0.16\textwidth]{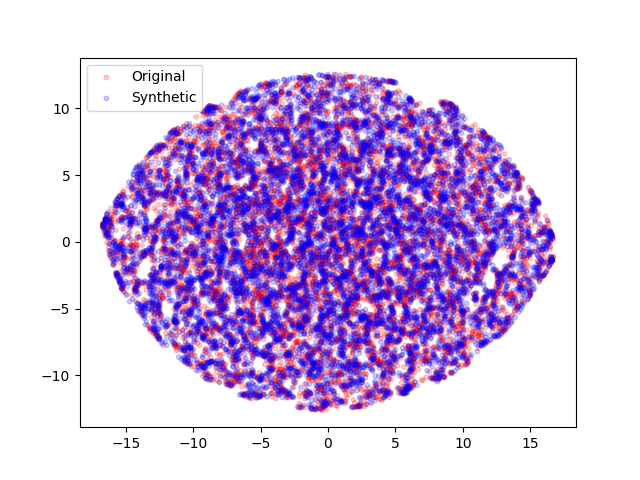}}
\hfill
\subfloat[Energy]{\includegraphics[width=0.16\textwidth]{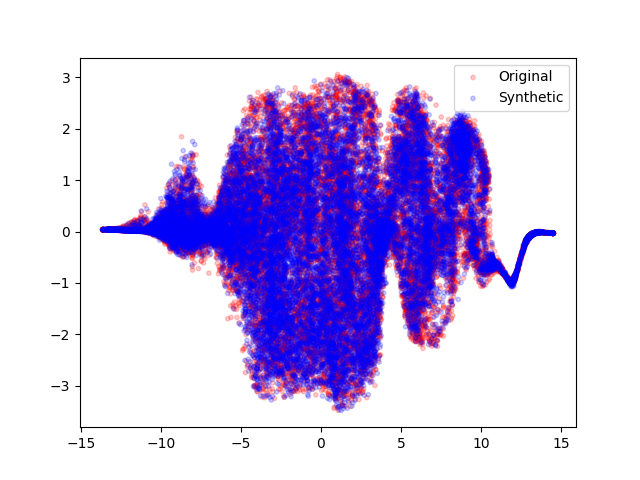}}
\hfill
\subfloat[fMRI]{\includegraphics[width=0.16\textwidth]{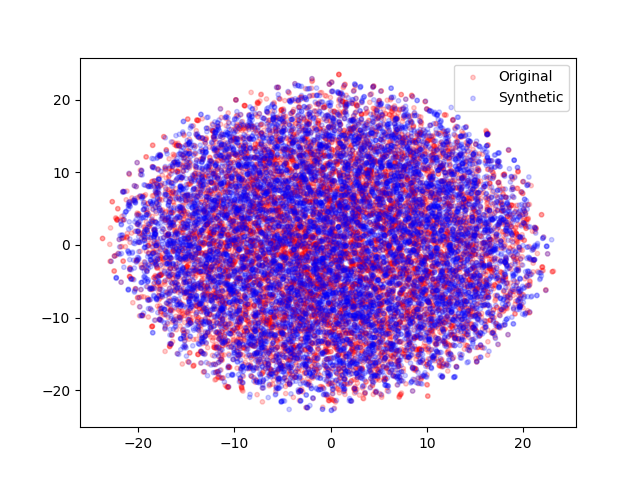}}

\caption{t-SNE visualizations of the time series synthesized by MSDformer, SDformer and Diffusion-TS.}
\label{fig:vis_tsne}

\end{figure*}

%% file: figure/vis_kernel.tex
\begin{figure*}
\centering

\raisebox{0.2\height}{\rotatebox{90}{\small MSDformer}} 
\subfloat{\includegraphics[width=0.16\textwidth]{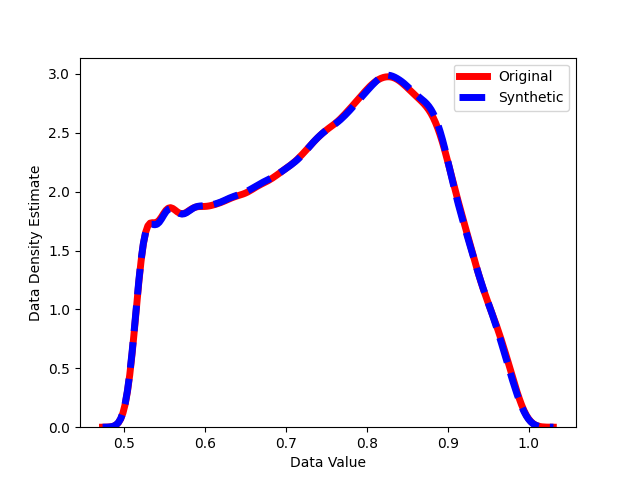}}
\hfill
\subfloat{\includegraphics[width=0.16\textwidth]{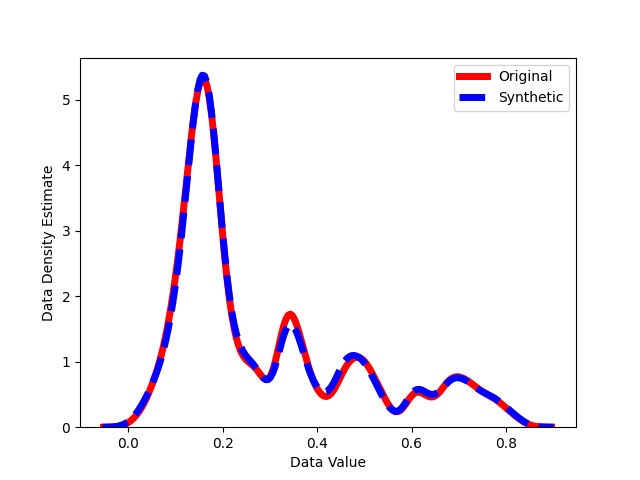}}
\hfill
\subfloat{\includegraphics[width=0.16\textwidth]{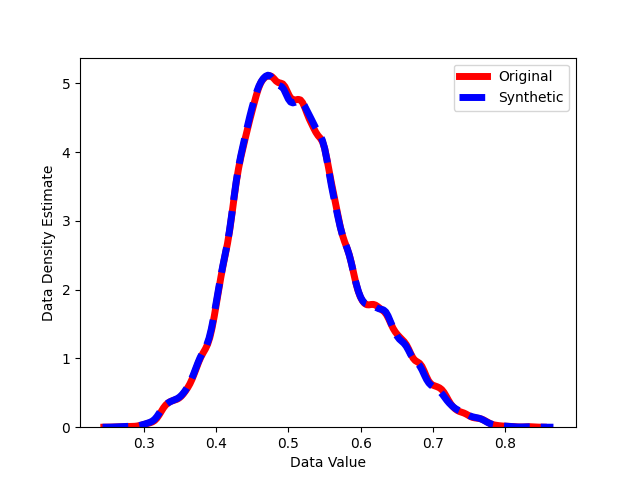}}
\hfill
\subfloat{\includegraphics[width=0.16\textwidth]{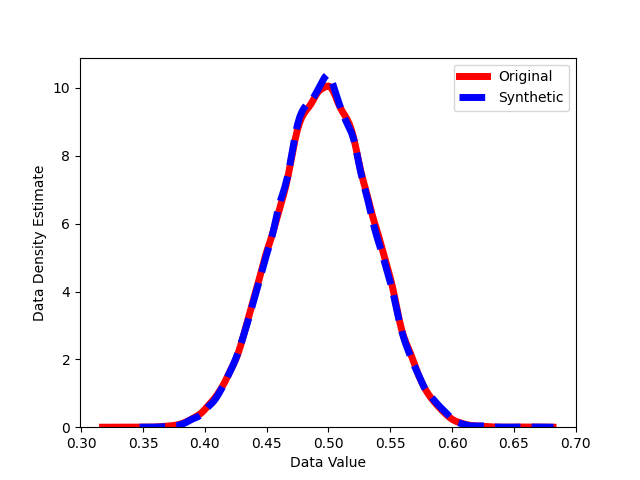}}
\hfill
\subfloat{\includegraphics[width=0.16\textwidth]{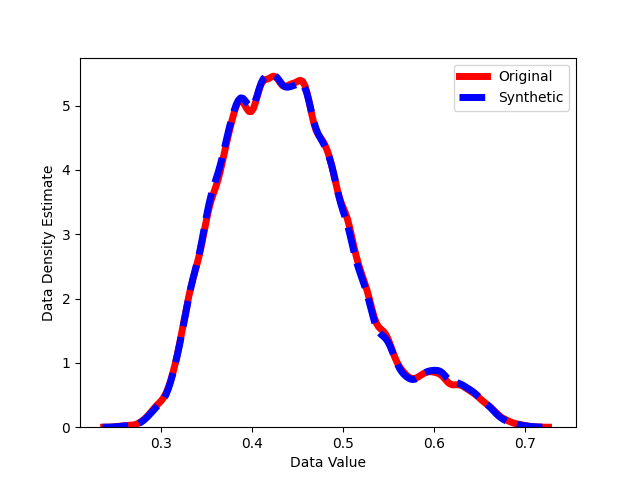}}
\hfill
\subfloat{\includegraphics[width=0.16\textwidth]{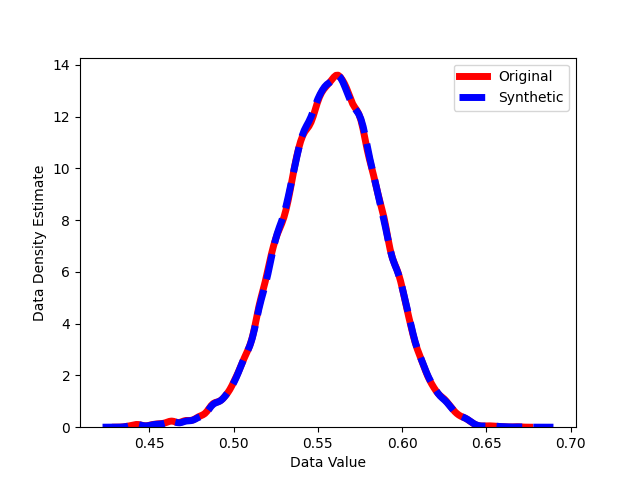}}

\vspace{-4mm}
\par
\raisebox{0.4\height}{\rotatebox{90}{\small SDformer}} 
\subfloat{\includegraphics[width=0.16\textwidth]{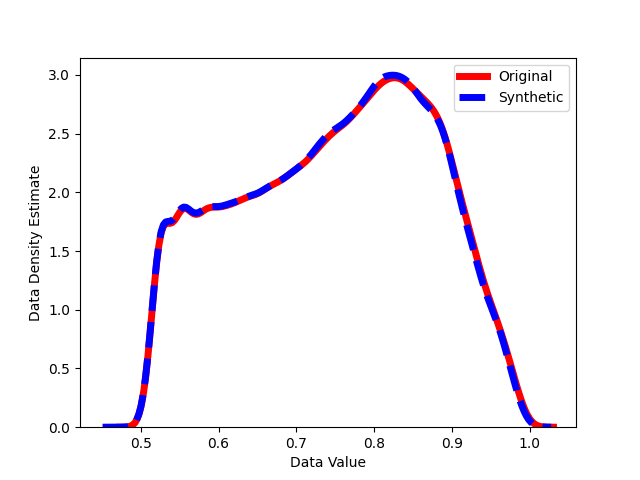}}
\hfill
\subfloat{\includegraphics[width=0.16\textwidth]{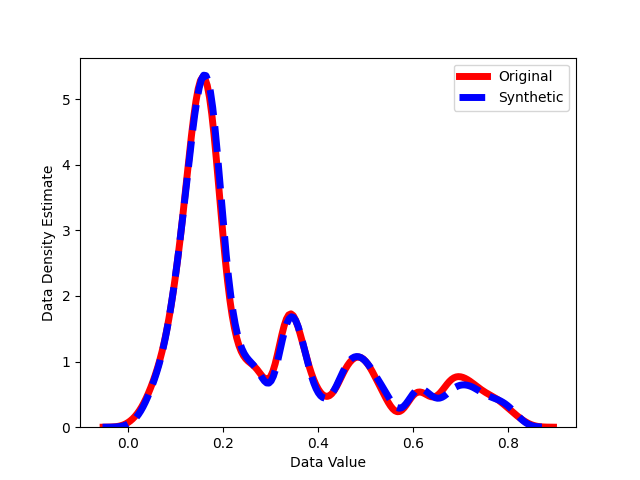}}
\hfill
\subfloat{\includegraphics[width=0.16\textwidth]{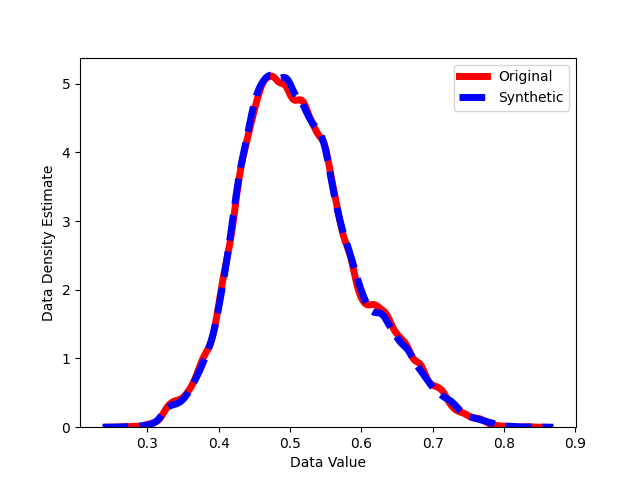}}
\hfill
\subfloat{\includegraphics[width=0.16\textwidth]{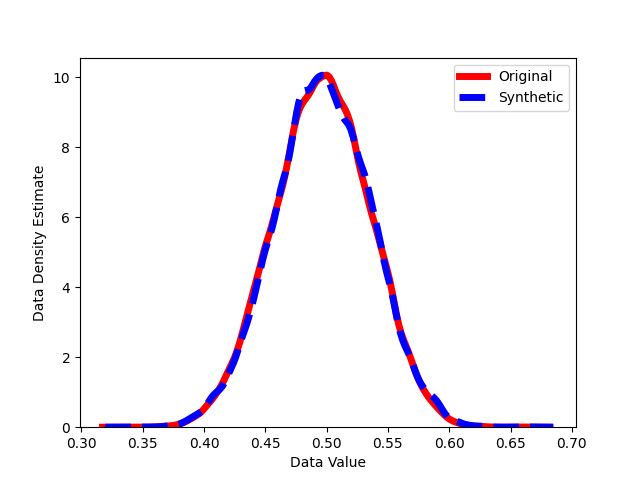}}
\hfill
\subfloat{\includegraphics[width=0.16\textwidth]{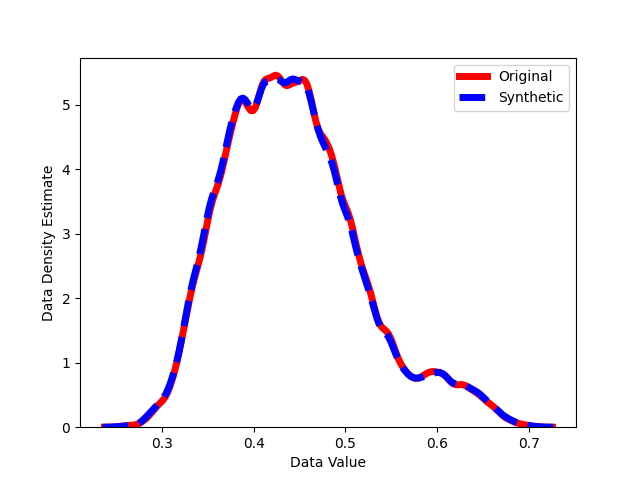}}
\hfill
\subfloat{\includegraphics[width=0.16\textwidth]{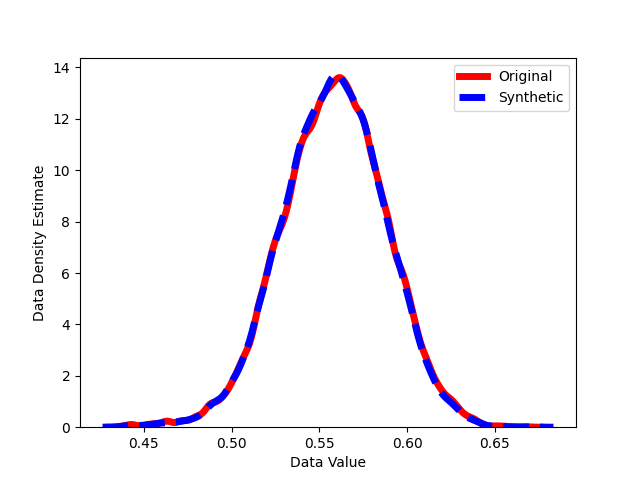}}

\vspace{-4mm}
\par
\setcounter{subfigure}{0}
\raisebox{0.1\height}{\rotatebox{90}{\small Diffusion-TS}} 
\subfloat[Sines]{\includegraphics[width=0.16\textwidth]{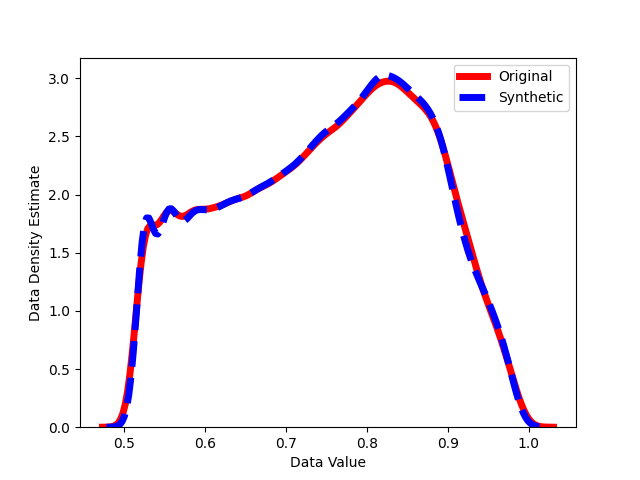}}
\hfill
\subfloat[Stocks]{\includegraphics[width=0.16\textwidth]{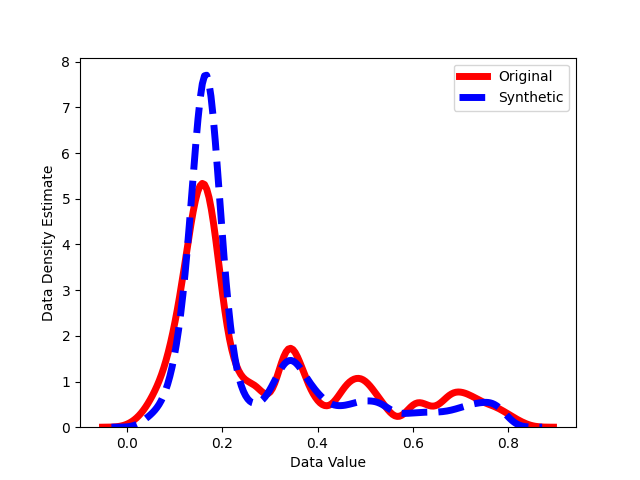}}
\hfill
\subfloat[ETTh]{\includegraphics[width=0.16\textwidth]{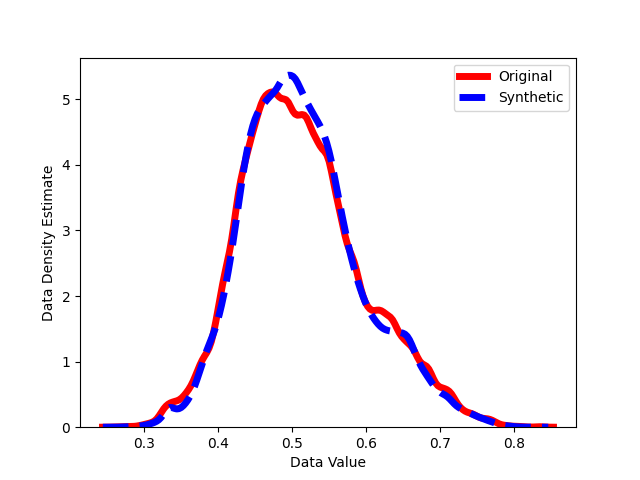}}
\hfill
\subfloat[MuJoCo]{\includegraphics[width=0.16\textwidth]{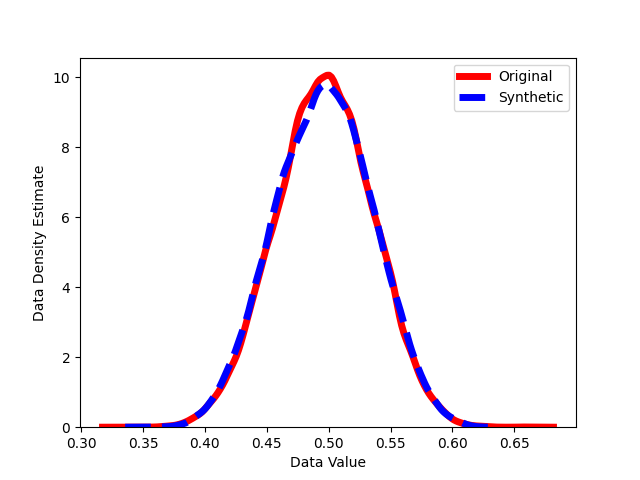}}
\hfill
\subfloat[Energy]{\includegraphics[width=0.16\textwidth]{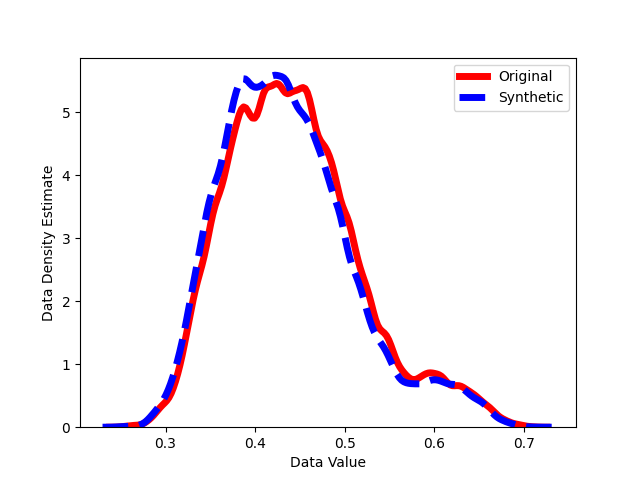}}
\hfill
\subfloat[fMRI]{\includegraphics[width=0.16\textwidth]{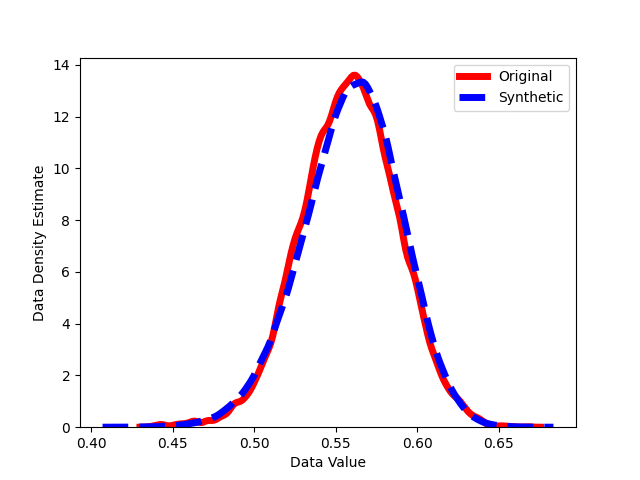}}

\caption{Kernel density estimation visualizations of the time series synthesized by MSDformer, SDformer and Diffusion-TS.}
  \label{fig:vis_kernel567}
\end{figure*}